\documentclass[twoside]{article}
\usepackage[accepted]{arxiv}

\usepackage[round]{natbib}

\usepackage{mathtools}
\usepackage{amsmath}
\usepackage{amssymb}
\usepackage{amsthm}
\usepackage{booktabs}
\usepackage{makecell}
\usepackage[hidelinks]{hyperref}
\usepackage{color, colortbl}
\usepackage[table]{xcolor}
\usepackage{tabularx}

\usepackage{xr}
\usepackage[capitalize]{cleveref}

\DeclareMathOperator{\EX}{\mathbb{E}}
\DeclareMathOperator{\R}{\mathbb{R}}
\DeclareMathOperator{\Ll}{\mathcal{L}}

\newtheorem{theorem}{Theorem}[section]
\newtheorem{rem}{Remark}[section]
\newtheorem{example}{Example}[section]

\usepackage{pifont}
\definecolor{customred}{RGB}{215,48,39}
\definecolor{customgreen}{RGB}{26,152,80}
\definecolor{customgray}{gray}{0.9}
\newcommand{\cmark}{\textcolor{customgreen}{\ding{51}}}
\newcommand{\xmark}{\textcolor{customred}{\ding{55}}}

\begin{document}



\twocolumn[

\aistatstitle{TRADE: Transfer of Distributions between External Conditions with Normalizing Flows}

\runningauthor{Stefan Wahl, Armand Rousselot, Felix Draxler, Henrik Schopmans, Ullrich Köthe}

\aistatsauthor{ Stefan Wahl$^{1, \star}$ \And Armand Rousselot$^{1, \star, \dagger}$ \And Felix Draxler$^1$ \And Henrik Schopmans$^{2}$ \And Ullrich Köthe$^1$ }

\aistatsaddress{  $^1$ Computer Vision and Learning Lab, Heidelberg University\\ 
$^2$ Institute of Theoretical Informatics, Karlsruhe Institute of Technology \\
$^\star$Equal contribution\\ $^\dagger$Corresponding author: armand.rousselot@iwr.uni-heidelberg.de } ]

\begin{abstract}
Modeling distributions that depend on external control parameters is a common scenario in diverse applications like molecular simulations, where system properties like temperature affect molecular configurations. Despite the relevance of these applications, existing solutions are unsatisfactory as they require severely restricted model architectures or rely on energy-based training, which is prone to instability. We introduce TRADE, which overcomes these limitations by formulating the learning process as a boundary value problem. By initially training the model for a specific condition using either i.i.d.~samples or backward KL training, we establish a boundary distribution. We then propagate this information across other conditions using the gradient of the unnormalized density with respect to the external parameter. This formulation, akin to the principles of physics-informed neural networks, allows us to efficiently learn parameter-dependent distributions without restrictive assumptions.
Experimentally, we demonstrate that TRADE achieves excellent results in a wide range of applications, ranging from Bayesian inference and molecular simulations to physical lattice models. Our code is available at \url{https://github.com/vislearn/trade}
\end{abstract}

\section{INTRODUCTION}
\label{sec:intro}

Generative models have achieved impressive performance in scientific applications among many other fields \citep{noe2019boltzmann, butter2022generative, cranmer2020frontier, sanchez2018inverse}. Oftentimes, such systems can depend on external control parameters, such as temperature governing the behavior of thermodynamic systems, coupling constants in physical models, or a tempered likelihood posterior in Bayesian inference \citep{wuttke2014temperature, friel2008marginal, Pawlowski_2017}. A major challenge in such cases is acquiring training data for a complete range of control parameters, which can quickly become infeasible.

\begin{figure}
    \centering
    \includegraphics[width=1.0\linewidth]{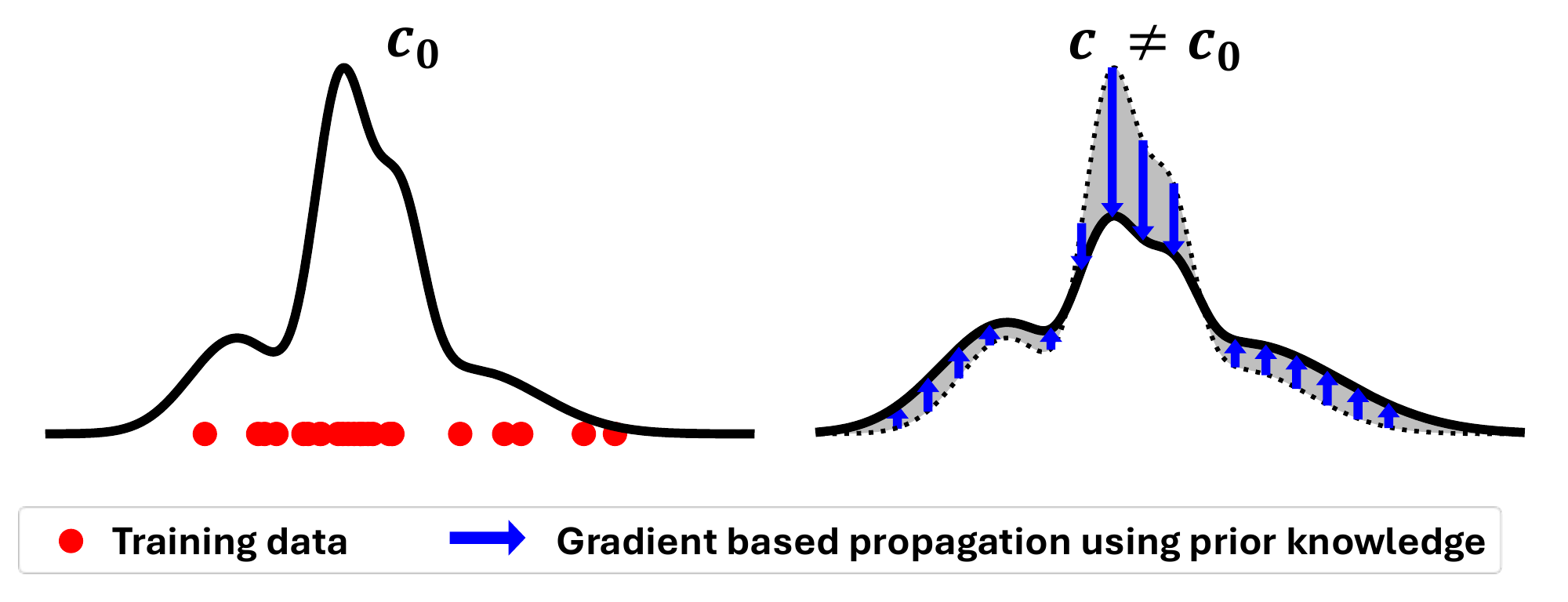}
    \caption{Our approach to train a conditional normalizing flow $p_\theta(x|c)$. \textbf{Left:} At $c=c_0$, the flow is trained using NLL. \textbf{Right:} By learning the gradient of the distribution with respect to $c$ based on prior knowledge, the distribution learned at $c_0$ is propagated to other conditions $c \neq c_0$ without additional training data.}
    \label{fig:MethodOverview}
\end{figure}
In this work, we focus on the common task of learning generative models in the case where we have access to the functional form of the unnormalized density, such as learning the distributions of equilibrium states in physical systems like molecules (Boltzmann distributions) or variational inference. We approach this problem by learning a single generative model that takes the external condition $c$ as a parameter: $p_\theta(x|c) \approx p(x|c)$.

Several works have attempted to address this problem before. One approach applies architectural restrictions to allow one particular functional form of external parameter dependence \citep{dibak2022temperature}. However, this restriction has recently been shown to incur severe limitations in expressivity \citep{draxler2024universality}. Energy-based training has been proposed as another method \citep{schebek2024efficient, invernizzi2022skipping, Wirnsberger_2020}, but can exhibit unfavorable properties like mode-seeking behavior \citep{minka2005divergence, felardos2023designing}, which has also been shown to be a problem in practice \citep{Wirnsberger_2020} (see \cref{app:GMM} for an example). Other works require data to be available at the target parameters \citep{wang2022data, Wirnsberger_2020}, which can become prohibitively expensive. We overcome these limitations: We allow learning arbitrary continuous dependencies of the target density on external parameters and train unconstrained architectures. Our central idea is to formulate the training of a conditional probability density as a boundary value problem: The boundary is learned on a fixed reference condition $c=c_0$, and then the resulting density is extrapolated using the known functional dependence on the condition underlying the problem. Our approach is summarized in \cref{fig:MethodOverview}.

In summary, we contribute:
\begin{itemize}
    \item We introduce TRADE, a method for learning generative models with arbitrary continuous dependencies on external conditions. Learning is enabled by a novel boundary value problem formulation.
    \item TRADE uses unconstrained architectures, facilitating the application to more complex target distributions.
    \item TRADE can be trained data-free or with data only available at the boundary $c_0$, making it an efficient approach in cases where acquiring data for a full range of control parameters is infeasible.
    \item TRADE achieves excellent results in a wide range of experiments including Bayesian inference, molecular simulations and physical lattice models
\end{itemize}

\section{RELATED WORK}
\label{sec:RelatedWork}
A line of work which is closely related to ours are physics-informed neural networks (PINNs), which were first introduced by \citet{RAISSI2019686}. The idea is to represent the solution of a PDE via a neural network, where the PDE is enforced via the loss function. The boundary values are usually given in the form of limited or noisy data. This approach has been successfully applied to a variety of scientific problems, including fluid dynamics, heat transfer, and material science \citep{cuomo2022scientific}. The works of \citet{liu2023pinf, guo2022normalizing} combine PINNs with normalizing flows for solving SDEs like the Fokker-Planck equation. \citet{cena2024physics} use a physics-informed loss to improve normalizing flow performance for detecting satellite power system faults. In our work, we similarly leverage a PDE constraint during training. However, our TRADE PDE is a function of the explicit dependence of the unnormalized distribution on an external control parameter. This is in contrast to standard PINNs, which typically enforce physical laws on the generated data itself. By embedding the control parameter directly into the PDE, TRADE is able to extrapolate to new control parameters even when only data from a single one is available.

A special case of distribution transfer is the generalization of a physical system across multiple thermodynamic states. The approach proposed by temperature steerable flows (TSF) \citep{dibak2022temperature} learns a transformation of samples with constant Jacobian determinant, which makes temperature scaling of the learned data distribution equivalent to temperature scaling of the latent distribution. This method is restricted to temperature scaling and, as shown in \citet{draxler2024universality}, such a transformation is severely limited in expressivity by design. This stands in contrast to TRADE, which introduces no additional restrictions on the model architecture. We showcase the effect of these architectural restrictions in \cref{app:GMM}.

Conditional Boltzmann generators (BG) \citep{schebek2024efficient} and learned replica exchange (LREX) \citep{invernizzi2022skipping} map samples from one thermodynamic state to another by training the model with backward KL, which can lead to problems of mode seeking behavior \citep{minka2005divergence, felardos2023designing}. Our method not only avoids these problems, but can also be trained knowing only the dependence of the energy function on the external parameter. This can even be possible in cases where computing the full ground truth energy of the system is infeasible, while the backward KL training objective requires this computation. Temperature-annealed Boltzmann generators \citet{schopmans2025temperature} tackle the task of training entirely data free at any target temperature. To avoid problems of backward KL training, which predominantly occur at low temperatures, they establish a boundary distribution at high temperatures using backward KL, then iteratively transfer it to the desired temperature by re-training with reweighted samples from the Boltzmann generator. As discussed in the next section, reweighting to low temperatures may lead to inefficient training due to low weights. Using TRADE to attempt to reduce the number of re-training steps could be an interesting future direction of research.

Learned Bennett acceptance ratio (LBAR) \citep{Wirnsberger_2020} learns to transfer distributions between two thermodynamic states to increase the efficiency of importance sampling, by using data from target thermodynamic state. Using data from multiple thermodynamic states, \citep{wang2022data} use diffusion models to interpolate between different temperatures and sample from the full range. Both of these methods require data at or around the target thermodynamic states. In contrast, TRADE has the ability to extrapolate from only a single thermodynamic state to others, by leveraging the information gained from the functional dependence of the unnormalized distribution on the control parameter.

\section{BACKGROUND}
\label{sec:NormalizingFlows}

Normalizing flows are a class of generative models that learn an invertible mapping $f_\theta: \R^d \mapsto \R^d$ mapping samples from the data distribution $p(x)$ to a tractable latent distribution $p_z(z)$, typically a Gaussian distribution \citep{kobyzev2020normalizing}. New samples from the model can be obtained by passing samples of the latent distribution through the inverse function $f_\theta^{-1}$. The model likelihood can be obtained by the change of variables formula
\begin{equation}
    \log p_\theta(x) = \log p_z(f_\theta(x)) + \log \left| \frac{\partial f_\theta(x)}{\partial x} \right|.
    \label{eq:CoV}
\end{equation}
Here, $\frac{\partial f_\theta(x)}{\partial x} \in \R^{d \times d}$ is the Jacobian matrix of $f_\theta$ and $| \cdot |$ denotes the absolute value of its determinant. In general, normalizing flows are trained by minimizing the negative log-likelihood (NLL) of the training data, which is equivalent to minimizing the (forward) KL divergence between data and model distribution
\begin{equation}
    D_\text{KL}(p(x) \| p_\theta(x))
    = \mathbb E_{x \sim p(x)}[\log p(x) - \log p_\theta(x)].
    \label{eq:ForwardKL}
\end{equation}
For typical neural networks, computing the Jacobian determinant during training can quickly become intractable. There are several ways this problem can be addressed. Some works define a surrogate objective that does not require the full Jacobian \citep{draxler2024free, sorrenson2024lifting}. Other works construct neural networks using so-called coupling blocks \citep{dinh2014nice, dinh2016density}, whose structure makes the network Jacobian determinant and inverse directly tractable. We choose the latter approach for this work, as it has the advantage of allowing for a tractable computation of the likelihood during training. 

Normalizing flows can also be trained in the absence of data \citep{felardos2023designing} if an (unnormalized) ground truth density $q(x) \propto p(x)$ is known, the two simplest methods being backward KL training and reweighting. Backward KL maximizes the likelihood of model samples under $p(x)$
\begin{equation}
    D_\text{KL}(p_\theta(x) \| p(x))
    \propto \mathbb E_{p_\theta(x)}[\log p_\theta(x) - \log q(x)],
    \label{eq:BackwardKL}
\end{equation}
which can suffer from problems like mode seeking behavior \citep{minka2005divergence, felardos2023designing, Wirnsberger_2020}. Reweighting uses samples from another distribution $\hat{p}$ and reweights them to simulate forward KL training with $p(x)$
\begin{align}
    %\nonumber
    D_\text{KL}(p(x) \| p_\theta(x)) 
    \propto
    \EX_{\hat{p}(x)}\left[  \frac{q(x)}{\hat{p}(x)} \log \left(\frac{q(x)}{p_\theta(x)} \right) \right],
    \label{eq:ReweightedKL}
\end{align}
but can suffer from unstable or inefficient training resulting from high and low weights respectively \citep{agrawal2020advances}. We illustrate failure cases of both backward KL training and reweighting in \cref{app:GMM}.

\section{METHOD}
\label{sec:Method}
\subsection{Conditional Distribution As Partial Differential Equation}
\label{sec:ConditionalDistributionAsPDE}

We are interested in approximating conditional probability distributions:
\begin{align}
    p_\theta(x|c) \approx p(x|c) = q(x|c) \frac{1}{\int q(x|c) dx},
\end{align}
where $q(x|c)$ is an unnormalized density and $p(x|c)$ its normalized variant.

We consider the case where (a)~we can successfully train our model for some initial $p(x|c_0)$, for example using i.i.d.~samples or via backward KL training (see \cref{sec:NormalizingFlows}), and (b)~we have access to the dependence of the \textit{unnormalized} $q(x|c)$ on $c$, as it is for example the case in physical systems where the temperature $T$ regulates the relative probability of states in the form of the Boltzmann distributions:
\begin{equation}
    q(x|T) = e^{-\frac{E(x)}{k_B T}}.
    \label{eq:BoltzmannDistribution}
\end{equation}
$k_B$ is the Boltzmann constant.

This allows treating $p_\theta(x|c)$ as the solution to the following partial differential equation (PDE):
\begin{align}
    p_\theta(x|c_0) &= p(x|c_0), \label{eq:BoundaryDensity} \\
    \frac{\partial}{\partial c} p_\theta(x|c) &= \frac{\partial}{\partial c} p(x|c). \label{eq:ConditionGradient}
\end{align}
In \cref{app:Proofs} we show that this boundary value problem has a unique solution. Intuitively, \cref{eq:BoundaryDensity} ensures that the correct distribution is learned at $c=c_0$. \Cref{eq:ConditionGradient} then propagates this information according to the gradient of the density. As noted in \cref{sec:NormalizingFlows}, the distribution at $p_\theta(x|c_0)$ can be learned from samples (forward KL) or using the unnormalized density $q(x|c_0)$. This ensures \cref{eq:BoundaryDensity}.

It is left to align the gradients of $p_\theta(x|c)$ with respect to $c$ to the gradients of $p(x|c)$ to fulfill \cref{eq:ConditionGradient}. Using the unnormalized density, we find:
\begin{equation}
    \frac{\partial}{\partial c} \log p(x|c) = \frac{\partial}{\partial c} \log q(x|c) - \frac{\partial}{\partial c} \log \int q(x|c) dx.
\end{equation}
We enable computing the latter term, the derivative of the normalization constant, via the following result:
\begin{theorem}
    \label{thm:ConditionGradientUnormalized}
    Given an \emph{unnormalized} density $q(x|c)$ that is differentiable in $c$, the derivative of the \emph{normalized} density $p(x|c)$ with respect to $c$ reads:
    \begin{equation}
        \frac{\partial}{\partial c}\log p(x|c)  
        =\frac{\partial }{\partial c}\log q(x|c)-\mathbb E_{p(x|c)}\!\left[\frac{\partial}{\partial c}\log q(x|c)\right]
        \label{eq:ConditionGradientUnormalized}
    \end{equation}
\end{theorem}
This reformulation allows efficient training with \cref{eq:ConditionGradient} via the functional form of $q(x|c)$, which is tractable in many cases. For example, in the case of the Boltzmann distribution in \cref{eq:BoltzmannDistribution}:
\begin{equation}\label{eq:DerivativeBoltzmannDistribution}
    \frac{\partial}{\partial T} \log q(x|T) = \frac{E(x)}{k_B T^2}
\end{equation}

There are several choices to estimate the expectation value $\mathbb E_{p(x|c)}[\cdot]$, which we compare later in \cref{sec:Ablation}. The proof of \cref{thm:ConditionGradientUnormalized} is straightforward:
\begin{proof}
    Compute the derivative of the true normalized distribution with respect to the condition. We write the partition function as $Z(c) = \int q(x|c) dx$:
    \begin{align}
        \frac{\partial}{\partial c} \log p(x|c) &
        %= \frac{\partial}{\partial c} \left[\log q(x|c) - \log Z(c)\right] \\&
        = \frac{\partial}{\partial c} \log q(x|c) - \frac{\partial}{\partial c} \log Z(c).\label{eq:Derivative-Ground-Truth-Proof-Theorem-1}
    \end{align}
    The only term left to consider is the derivative of the normalization as a function of the condition:
    \begin{align}
        \frac{\partial}{\partial c} \log Z(c) &
        = \frac{1}{Z(c)} \frac{\partial}{\partial c} \int q(x|c) dx \\&
        \overset{(*)}{=} \frac{1}{Z(c)} \int \frac{\partial}{\partial c} q(x|c) dx \\&
        = \frac{1}{Z(c)} \int \frac{\partial}{\partial c} e^{\log q(x|c)} dx \\&
        = \frac{1}{Z(c)} \int e^{\log q(x|c)} \frac{\partial}{\partial c} \log q(x|c) dx \\&
        = \int p(x|c) \frac{\partial}{\partial c} \log q(x|c) dx. \label{eq:Derivative-Z-As-Expectation-Value-Proof-Theorem-1}
    \end{align}
    Combining \cref{eq:Derivative-Ground-Truth-Proof-Theorem-1,eq:Derivative-Z-As-Expectation-Value-Proof-Theorem-1} yields the result. See \cref{app:SwapDerivativeIntegral} for why we can swap integral and derivative in $(*)$.
\end{proof}

\subsection{Training Objective}

From the PDE in \cref{eq:BoundaryDensity,eq:ConditionGradient}, we construct the following loss terms for training. First, we consider some suitable objective $d(\cdot, \cdot)$ which learns the density at the boundary (see \cref{sec:NormalizingFlows}):
\begin{align}
    \Ll_\text{boundary} = d(p(x|c_0), p_\theta(x|c_0)).
\end{align}

Second, we add a loss that fixes the derivatives of the learned density to follow \cref{eq:ConditionGradient}. We directly insert:
\begin{equation}\label{eq:GradientBasedLossContributionTrade}
    \begin{matrix*}[l]
        \Ll_\text{grad}(c, x) = \\
        \left\| \frac{\partial}{\partial c} \log \left(\frac{p_\theta(x|c)}{q(x|c)}\right) + \EX_{\tilde x \sim p(x|c)}\!\left[\frac{\partial}{\partial c} \log q(\tilde x|c)\right] \right\|^2.
    \end{matrix*}
\end{equation}
We evaluate the gradient loss $\Ll_\text{grad}(c, x)$ at points sampled from proposal distributions $p_{\operatorname{grad}}(c)$ (see \cref{sec:ConditionSampling}) and $p_{\operatorname{grad}}(x|c)$ (see \cref{sec:EvaluationPointSampling}).
% Note that the values for the conditions $c$ and points $x$ to evaluate $\Ll_\text{grad}(c, x)$ at is a hyperparameter, see \cref{sec:ImplementationChoices}.
To evaluate the expectation of the gradient w.r.t.~$c$, we discuss possible options in \cref{sec:ExpectationValueComputation}.% and include an ablation in \cref{sec:Ablation}.

Combining the two loss terms, we find the following differentiable optimization objective:
\begin{equation}
    \Ll = \Ll_\text{boundary} + \lambda \EX_{p_{\operatorname{grad}}(c, x)}[\Ll_\text{grad}(c, x)].
    \label{eq:FullLoss}
\end{equation}
Inserting \cref{eq:BoundaryDensity,eq:ConditionGradient} shows that this loss is indeed minimized for the true solution.

\subsection{Practical Implementation}\label{sec:practical-implementation}

The formulation of the loss in \cref{eq:FullLoss} leaves several design choices open. We propose best practices in this section and perform an ablation in \cref{sec:Ablation}.

\subsubsection{Discretization of the Condition}

The first choice is to decide whether one should discretize the range of possible conditions $C = [c_\text{min}, c_\text{max}]$ into a grid $\tilde{C}$ or sample from $C$ continuously. Discretization comes with the advantage of being able to save information from previous batches at each grid point and incorporate it into e.g.~the definition of $p_{\operatorname{grad}}(c)$ or a better estimation of $\EX_{p(x|c)}\!\left[\frac{\partial}{\partial c} \log q(x|c)\right]$, while continuous sampling avoids discretization errors. In the following we provide design choices for both methods. For low- to medium-dimensional problems we find that both methods can perform equally well apart from a small discretization error (see \cref{sec:Ablation}). For high-dimensional problem where the ground-truth energy $q(x|c)$ is known, the discretized loss might yield better results.

\subsubsection{Sampling of the Conditioning}
\label{sec:ConditionSampling}

The main idea behind our approach is that the flow learns the distribution explicitly only at the boundary $c=c_0$. The learned distribution is then propagated to $c \neq c_0$ through the gradient loss. Intuitively, this implies that the conditions near the boundary $c \approx c_0$ should be learned first, and the longer the training runs, the wider the interval of conditions should be.

Using the discretized grid of conditions $\tilde{C}$, we can apply a scheme inspired by the temporal causality weights proposed in \citet{wang2022respectingcausalityneedtraining}. It samples conditions $c$ weighted by the accumulated gradient loss from the reference point $c_0$:
\begin{align}
    \label{eq:CausalityWeightsDefinition}
    p_\text{grad}(c) &\propto \exp\left( -\epsilon \int_{c_0}^{c} \Ll_\text{grad}(\tilde c) d \tilde c \right) \\&
    \approx \exp\left( -\epsilon \sum_{c_i \in [c_0, c] \cap \tilde{C}} \Ll_\text{grad}(c_i ) \right)
\end{align}
If $\Ll_\text{grad}(c')$ is large for some $c'$, then the conditions further away from $c_0$ than $c'$ are weighted down.
We keep track of the loss values at the different conditions $\Ll_\text{grad}(c_i)$ during training using exponential averaging.

Without discretization, the integral in \cref{eq:CausalityWeightsDefinition} is intractable. In this case we adapt a different schedule to sample $c$, based on the current training step $t$, which we describe in \cref{app:ContinuousCSampling}.

\subsubsection{Sampling of Evaluation Points}
\label{sec:EvaluationPointSampling}

The next design decision is at which points $x$ to evaluate the gradient loss. By \cref{eq:ConditionGradient}, the gradient should follow the target distribution everywhere on the domain. In principle, we could therefore choose the proposal distribution $p_{\operatorname{grad}}(x|c)$ arbitrarily as long as it captures the domain.

In practice, we find that choosing the model distribution $p_{\operatorname{grad}}(x|c) = p_\theta(x|c)$ offers a good trade-off between exploration and exploitation. A small improvement can be achieved by perturbing the resulting samples with a small amount of Gaussian noise. In addition, we find that using prior knowledge about the sampling process, such as symmetries, can be beneficial.% \todo{We also find that it can be beneficial to use prior knowledge about the sampling process, such as symmetries, to achieve X.}

% The computation of the residuals \cref{eq:residualPILoss} requires evaluation points. The formulation of the loss does not tell use how the choose these points. Our default is to simply sample points following the flow distribution at the evaluated inverse temperature.However, we find, that it can be beneficial to apply prior knowledge to the sampling process. One example are known symmetries.

% \begin{itemize}
%     \item What are normal PINNs
% \end{itemize}

\subsubsection{Computation of Expectation Values}
\label{sec:ExpectationValueComputation}

\Cref{thm:ConditionGradientUnormalized} allows to compute the derivative of the normalized density $\frac{\partial}{\partial c} p(x|c)$ via access to the derivative of the unnormalized density $\frac{\partial}{\partial c} q(x|c)$. However, it also involves its expectation over samples from $p(x|c)$:
\begin{equation}
    \bar\nabla(c) := \EX_{x \sim p(x|c)}\left[ \frac{\partial}{\partial c} \log q(x|c) \right].
\end{equation}
For example, in Boltzmann distributions, this corresponds to the average energy at the evaluated temperature rescaled by $1/(k_BT^2)$.

To compute this expectation,  we perform self-normalized importance sampling (SNIS) using $N$ samples from a base distribution $p_{\operatorname{base}}(x|c)$, together with the unnormalized density:
\begin{align}
    \bar\nabla(c) &= \frac{\EX_{x \sim p_{\operatorname{base}}(x|c)}\left[ \frac{q(x|c)}{p_{\operatorname{base}}(x|c)} \frac{\partial}{\partial c} \log q(x|c) \right]}{\EX_{x \sim p_{\operatorname{base}}(x|c)}\left[ \frac{q(x|c)}{p_{\operatorname{base}}(x|c)} \right]} \\&
    \approx \frac{\sum_{i=1}^N \frac{q(x_i|c)}{p_{\operatorname{base}}(x_i|c)} \frac{\partial}{\partial c} \log q(x_i|c)}{\sum_{i=1}^N \frac{q(x_i|c)}{p_{\operatorname{base}}(x_i|c)} }
\end{align}
The numerator estimates the gradients, and the denominator correctly scales it with the estimated normalization constant of $q(x|c)$. Although SNIS performs well in our experiments and converges to the true expectation value, the estimator remains biased for finite $N$ with a bias of $\mathcal{O}\left(\frac{1}{N}\right)$ (see, e.g., \citet{cardoso2022brsnisbiasreducedselfnormalized}). In this scheme we can use any base distribution which covers the domain. In \cref{sec:Ablation}, we test different choices for $p_{\operatorname{base}}(x|c)$ and again find that using the current model estimate $p_\theta(x|c)$ of the distribution at parameter $c$ generally works best.

When using the discretized domain $\tilde{C}$, we find that evaluating the expectation values for every point on the grid in specified intervals and taking a running exponential average allows for using much larger sample sizes $N$ and more accurate expectation value estimates at the same overall computational cost.

\subsubsection{Energy-Free Training of TRADE}
\label{sec:EnergyFreeTraining}

In some cases, only the functional form of the derivative $\frac{\partial q(x|c)}{\partial c}$ may be known, but $q(x|c)$ itself contains a term that is unknown (e.g. \cref{sec:TwoMoons}). In this case, we can often estimate this unknown term from the current model, via $q(x|c) \approx p_\theta(x|c)$.

For example, for performing temperature scaling of a Boltzmann distribution we are interested in the derivative in \cref{eq:DerivativeBoltzmannDistribution}.

If $E(x)$ is unknown or computationally expensive, we can substitute $E(x) \approx - k_BT_0\log p_\theta(x|T_0)$. This means we can train TRADE only based on samples, without access to the ground truth density. This is impossible with methods that rely on backward KL training.

\section{EXPERIMENTS}
\label{sec:Experiments}

For all experiments presented in this section, experimental details are provided in \cref{app:ExperimentalDetails}.

\subsection{Ablation Study: Multidimensional Wells}
\label{sec:Ablation}

\begin{table*}[h!tb]
\caption{Ablation study over different methods and design choices on a multidimensional-well system in $d=5$ dimensions. We train models at temperature $T=1.0$ and evaluate their ability to transfer the distribution to $T=0.5$ in terms of NLL and ESS. Standard deviations are evaluated over 3 separate runs. TRADE (lower block) manages to achieve the best performance in comparison to several baseline methods (upper block). Gray rows mark configurations which were unstable. }
\label{tab:Ablation}
\centering
\resizebox{0.95\linewidth}{!}{
\begin{tabular}{cccccccc}
\toprule
\multicolumn{4}{c}{Baselines} & NLL T=0.5 $\downarrow$ & NLL T=1.0 $\downarrow$ & ESS T=0.5 $\uparrow$ & ESS T=1.0 $\uparrow$ \\
\midrule
\multicolumn{4}{c}{Forward KL training at $T=1.0$ only} & 2.81 ± 0.0 & 4.37 ± 0.0 & 42.2 ± 0.2 \% & 98.8 ± 0.0 \% \\
\multicolumn{4}{c}{Backward KL training} & 2.29 ± 0.01 & 4.40 ± 0.0 & 84.1 ± 0.3 \% & 92.3 ± 0.2 \% \\
\multicolumn{4}{c}{Training with reweighted samples} & 2.23 ± 0.0 & 4.37 ± 0.0 & 91.3 ± 1.1 \% & 98.9 ± 0.0 \% \\
\multicolumn{4}{c}{Temperature-Steerable Flows \citep{dibak2022temperature}} & 2.45 ± 0.02 & 4.45 ± 0.01 & 69.5 ± 1.7 \% & 81.7 ± 4.8 \% \\
\midrule
Discretize & $p_{\operatorname{base}}(x|c)$ & \makecell{Use ground\\truth $q(x|c)$} & \makecell{Causality\\weights} & &&&\\% NLL T=0.5 $\downarrow$ & NLL T=1.0 $\downarrow$ & ESS T=0.5 $\uparrow$ & ESS T=1.0 $\uparrow$ \\
\midrule
\rowcolor{customgray} \cmark & $p_\theta(x|c)$ & \cmark & \xmark & 11.3 ± 9.93 & 5.79 ± 0.26 & 0.0 ± 0.0 \% & 10.3 ± 2.3 \% \\
\rowcolor{customgray} \cmark & $p_\theta(x|c)$ & \xmark & \cmark & 5.15 ± 1.15 & 4.58 ± 0.14 & 0.2 ± 0.3 \% & 63.7 ± 19.6 \% \\
\rowcolor{customgray} \xmark & $p_\theta(x|c_0)$ & \xmark &   & 13.45 ± 4.56 & 9.2 ± 2.07 & 0.1 ± 0.1 \% & 0.6 ± 0.5 \% \\
\xmark & $p(x|c_0)$ & \cmark &   & 2.57 ± 0.04 & 4.63 ± 0.02 & 42.7 ± 0.02 \% & 59.9 ± 2.6 \% \\
\xmark & $p(x|c_0)$ & \xmark &   & 2.65 ± 0.03 & 4.48 ± 0.01 & 34.5 ± 3.5 \% & 78.4 ± 1.4 \% \\
\xmark & $p_\theta(x|c_0)$ & \cmark &   & 2.27 ± 0.0 & 4.38 ± 0.0 & 81.1 ± 0.4 \% & 97. ± 0.1 \% \\
\xmark & $p_\theta(x|c)$ & \cmark &   & \textbf{2.22} ± 0.0 & \textbf{4.36} ± 0.0 & \textbf{96.7} ± 1.1 \% & \textbf{99.3} ± 0.1 \% \\
\xmark & $p_\theta(x|c)$ & \xmark &   & \textbf{2.22} ± 0.0 & 4.37 ± 0.0 & \textbf{95.6} ± 2.3 \% & \textbf{99.2} ± 0.1 \% \\
\cmark & $p_\theta(x|c)$ & \cmark & \cmark & 2.25 ± 0.0 & 4.37 ± 0.0 & 89.6 ± 0.5 \% & 97.9 ± 0.1 \% \\
\bottomrule
\end{tabular}
}
\end{table*}

We perform an ablation study over the different design choices listed in the previous section and compare TRADE against training with backward KL, reweighted samples and TSF \citep{dibak2022temperature}. The task is to learn configurations $x$ of a multidimensional well system in $d=5$ dimensions, which consists of $2^5 = 32$ basins. Configurations follow a Boltzmann distribution (\cref{eq:BoltzmannDistribution}) defined by the following energy function:
\begin{equation}
    E(x) = \sum\limits_{i=1}^{d} a x_i + b x_i^2 + c x_i^4 .
    \label{eq:MultiwellEnergy}
\end{equation}
Training data is given to all models at $T_0=1.0$, with the goal being to learn the distribution for a range $T \in [0.5, 1.0]$. 
The derivative required for training TRADE is given by \cref{eq:DerivativeBoltzmannDistribution}.

For the discretization of $T$, we choose 15 grid points equidistantly in logarithmic space and use a periodically updated exponential running average for estimating the necessary expectation values as described in \cref{sec:ExpectationValueComputation}. While the backward KL model could in principle be trained entirely data-free, for the purpose of an equal comparison it is simultaneously trained with forward KL at $T=1.0$ using the same training data that all other models receive (this can be regarded as a similar training procedure to \citet{schebek2024efficient}). 
In \cref{tab:Ablation} we report the results of the all models in terms of NLL and (relative) effective sample size (ESS) \citep{kish1965sampling}.

The results provide a few guidelines for design choices in TRADE. Firstly, choosing $p_{\operatorname{base}}(x|c) = p_\theta(x|c)$ seems to always be optimal. Secondly, when using a non-discretized loss, replacing $q(x|c)$ with an approximation by the model (as described in \cref{sec:EnergyFreeTraining}) barely impacts model performance, while it severely limits TRADE's performance when using a discretized loss. However, when using causality weights and the ground truth $q(x|c)$ the performance loss by discretizing is comparatively low. Compared to existing methods for temperature transfer TRADE outperforms all of them both in terms of NLL and ESS.

\subsection{Tempered Bayesian Inference: Two Moons}

\begin{figure*}[h]
    \centering
    \includegraphics[width=0.32\linewidth]{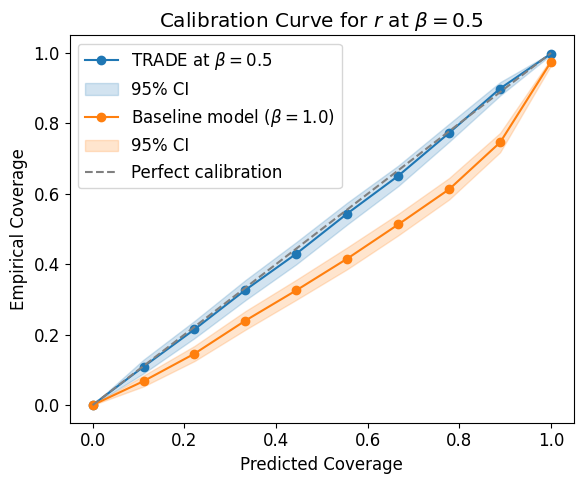}
    \includegraphics[width=0.32\linewidth]{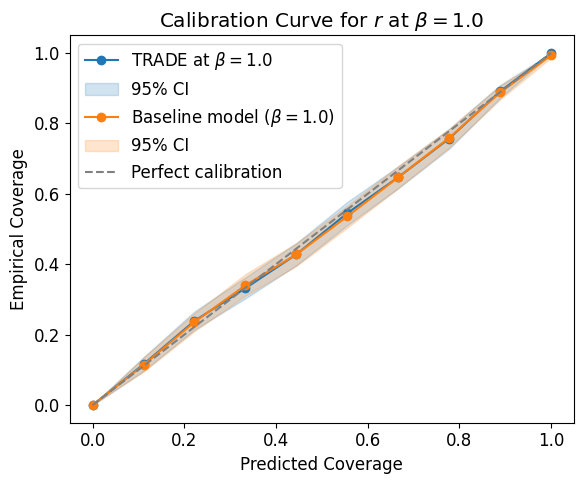}
    \includegraphics[width=0.32\linewidth]{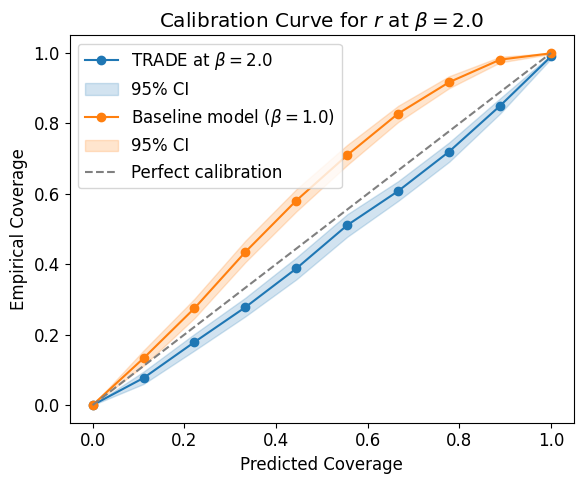}
    \caption{Calibration curves of the estimated $r$ (\cref{eq:TwoMoonsRDefinition}) at different likelihood powers $\beta$ for the two moons dataset. We compare TRADE to a model trained only at $\beta=1.0$ for $\beta \in \{0.5, 1, 2\}$. Our method successfully generalizes to different $\beta$ while maintaining the same performance at $\beta=1.0$. The model trained only at $\beta=1.0$ fails to generalize to other values of $\beta$.}
    \label{fig:TwoMoonsResults}
\end{figure*}

\label{sec:TwoMoons}
We apply TRADE in a Bayesian inference setting \citep{cranmer2020frontier}, where the task is to learn the posterior distribution of the underlying parameters of a system $\psi$ given observations $y$
\begin{equation}
    p(\psi | y) \propto p(\psi) p(y | \psi),
\end{equation}
where $p(\psi)$ is the prior and $p(y|\psi)$ is the likelihood of observation $y$ given a parameter $\psi$.
Tempered Bayesian inference learns a family of posteriors \citep{friel2008marginal} which is defined as follows:
\begin{equation}
    p(\psi | y, \beta) \propto p(\psi) p(y | \psi)^{\beta}.
    \label{eq:PowerPosterior}
\end{equation}
The parameter $\beta$ controls the trade-off between the prior and the data: $\beta=1$ recovers the standard posterior, while $\beta>1$ emphasizes the data and $\beta<1$ emphasizes the prior.
We train our model to perform tempered Bayesian inference controlled by $\beta$ on the two moons task, introduced in \citet{greenberg2019automatic}, a common benchmark for inference models. Observations $y$ given parameters $\psi$ are drawn as follows:
\begin{align}
    \label{eq:TwoMoonsRDefinition}
    & r \sim \mathcal{N}(0.1, 0.01^2) \\
    & \alpha \sim \mathcal{U}\big(-\frac{\pi}{2}, \frac{\pi}{2}\big) \\
    & y_1 = - \frac{| \psi_1 - \psi_2 |}{\sqrt{2}} + r\cos(\alpha) + 0.25 \\
    & y_2 = - \frac{- \psi_1 + \psi_2}{\sqrt{2}} + r\sin(\alpha)
    \label{eq:TwoMoonsDefinition}
\end{align}
As opposed to \citet{greenberg2019automatic}, we choose a non-uniform prior $p(\psi) = \mathcal{N}(0, 0.3^2)(\psi)$, making the task distinct from temperature scaling and preventing the application of TSF. The derivative needed for $\Ll_\text{grad}$ is the following:
\begin{equation}
    \frac{\partial}{\partial \beta}  \log q(\psi | y, \beta) \propto \log p(y | \psi)
    \label{eq:GradientTwoMoons}
\end{equation}

While we have a closed-form solution of the likelihood $\log p(y | \psi)$ in this case, in many other applications this might not be available. As discussed in \cref{sec:EnergyFreeTraining}, a key advantage of TRADE is that it can be trained without access to the ground truth density, estimating $q(x|c)$ by the normalizing flow. We showcase this capability by only using the known prior $p(\psi)$ and the learned posterior $p_\theta(\psi|y)$ to estimate the likelihood $p(y|\psi) \approx \frac{p_\theta(\psi|y)}{p(\psi)}$ at training time. 

In \cref{fig:TwoMoonsResults} we provide a comparison in terms of calibration of the estimated $r$ (\cref{eq:TwoMoonsRDefinition}) between a TRADE trained at $\beta \in [0.5, 2.0]$ compared to a model only using $\beta = 1.0$. At higher values of $\beta$, the tempered posterior of $r$ contracts and the baseline model becomes underconfident, and vice-versa. In contrast, it is apparent that TRADE learns to generalize to different $\beta$ values without incurring performance loss at $\beta = 1.0$.

\subsection{Temperature Scaling of Alanine Dipeptide}
\label{sec:ExperimentAlanineDipeptide}

We now demonstrate TRADE's ability to transfer data from molecular dynamics (MD) simulations between different temperatures. For this, we follow the setting from \citet{dibak2022temperature}. The goal is to learn the distribution of Alanine Dipeptide configurations $x$ between different temperatures $T \in [300K, 600K]$. As before, the configurations follow a Boltzmann distribution with energy $E(x)$, with the required derivative being \cref{eq:DerivativeBoltzmannDistribution}.

\Cref{tab:AlanineDipeptideResults} compares TRADE to TSF \citep{dibak2022temperature} and backward KL training. We train all models with MD simulation data at $600K$ from \citet{dibak2022temperature} and compare the resulting models at $600K$ and $300K$. For this we compute the negative log-likelihood (NLL) of the test data and the KL divergence between the 2D histograms ($50 \times 50$ bins) of model samples and test samples of the $\phi, \psi$ angles of the molecule (see \citet{dibak2022temperature} figure 3 for an illustration). 
For each method we train two models, one based on affine coupling blocks \citep{dinh2016density} and one based on the more expressive rational quadratic spline blocks \citep{durkan2019neuralsplineflows}. To construct volume-preserving versions for TSF we use the architectures given in \citet{dibak2022temperature}. For backward KL training we observe unstable training at the target temperature (especially in the affine model), while TRADE performs best with both architectures. It is worth noting that  TSF drops in performance when using spline couplings, which we suspect is caused by a high approximation error of the temperature scaled spline. In \cref{fig:AlanineDipeptidePlotSpline} we show the best of three experiments for all spline models at $T=300K$, which demonstrates TRADE's ability to extrapolate to lower temperatures accurately. Figures for affine models and other temperatures are provided in \cref{app:ExperimentalDetails}.

\begin{table*}
\centering
\resizebox{0.95\linewidth}{!}{
\begin{tabular}{ccccc}
\toprule
\multicolumn{1}{c}{Method} & \multicolumn{1}{c}{NLL $T=300K$ $\downarrow$} & \multicolumn{1}{c}{NLL $T=600K$ $\downarrow$} & \multicolumn{1}{c}{$D_{\text{hist}(\phi,\psi)}$ $T=300K$ $\downarrow$} & \multicolumn{1}{c}{$D_{\text{hist}(\phi,\psi)}$ $T=600K$ $\downarrow$} \\ \midrule
Backward KL Affine & -136.6067 ± 10.4910 & -143.5167 ± 0.3188 & 0.1726 ± 0.0599 & 0.0438 ± 0.0038\\
TSF Affine \citep{dibak2022temperature} & -159.1200 ± 0.1179 & -144.3500 ± 0.0200 & 0.1707 ± 0.0503 & \textbf{0.0331 ± 0.0040} \\ 
TRADE Affine (Ours) & \textbf{-159.8333 ± 0.0802} & \textbf{-144.4567 ± 0.0513} & \textbf{0.1402 ± 0.0290} & 0.0353 ± 0.0141 \\ \hline
Backward KL Spline & -160.7967 ± 0.4102 & -145.1533 ± 0.0737 & 0.3982 ± 0.4716 & 0.0321 ± 0.0252 \\
TSF Spline \citep{dibak2022temperature} & -157.3600 ± 0.2364 & -145.1400 ± 0.0000 & 0.1813 ± 0.0351 & 0.0200 ± 0.0019 \\
TRADE Spline (Ours) & \textbf{-161.1633 ± 0.0153} & \textbf{-145.1900 ± 0.0000} & \textbf{0.0291 ± 0.0025} & \textbf{0.0077 ± 0.0002} \\
\bottomrule
\end{tabular}
}
\caption{Comparison of different temperature scaling methods trained on MD simulations of Alanine Dipeptide at 600K. We train all models with two different coupling architectures, affine coupling blocks \citep{dinh2016density} and rational quadratic spline coupling blocks \citep{durkan2019neuralsplineflows}, using the architecture presented in \citet{dibak2022temperature} to construct an approximately temperature scalable network for TSF. We evaluate models in terms of negative log-likelihood (NLL) and KL divergence between the 2D histograms of the $\phi$, $\psi$ of model samples and MD samples both at temperatures $T=300K$ and $T=600K$. Standard deviations are taken over three independent runs with the same hyperparameters.}
\label{tab:AlanineDipeptideResults}
\end{table*}

\begin{figure*}
    \centering
    \includegraphics[width=\linewidth]{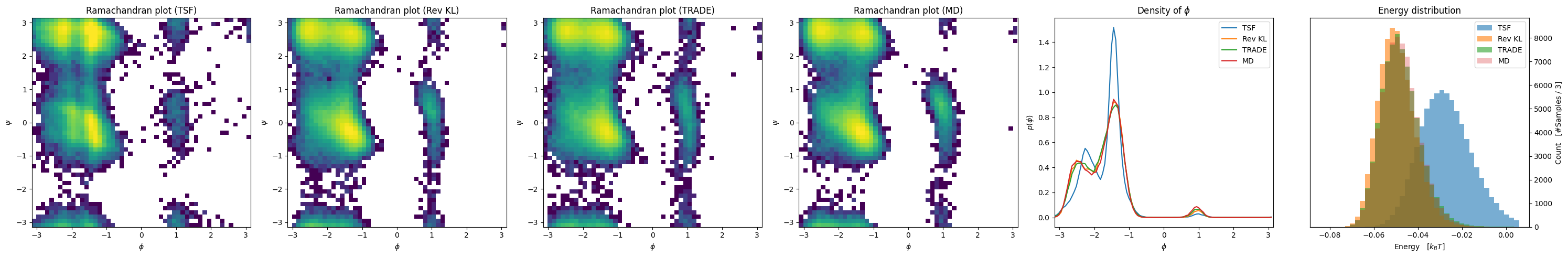}
    \caption{A comparison of TSF and backward KL training to TRADE trained on MD simulations of Alanine Dipeptide. Both models were trained at 600K and are evaluated at 300K. From left to right: Ramachandran plots of model samples (TSF, backward KL, TRADE) and molecular dynamics (MD), marginal density of the $\phi$ angle, ground truth energy of model and MD samples. For each model we plot the best result of three runs.}
    \label{fig:AlanineDipeptidePlotSpline}
\end{figure*}

\subsection{Varying External Parameters in a Scalar Field Theory}\label{sec:ExperimentScalarTheory}

\begin{figure}
    \centering
    \includegraphics[width=0.925\linewidth]{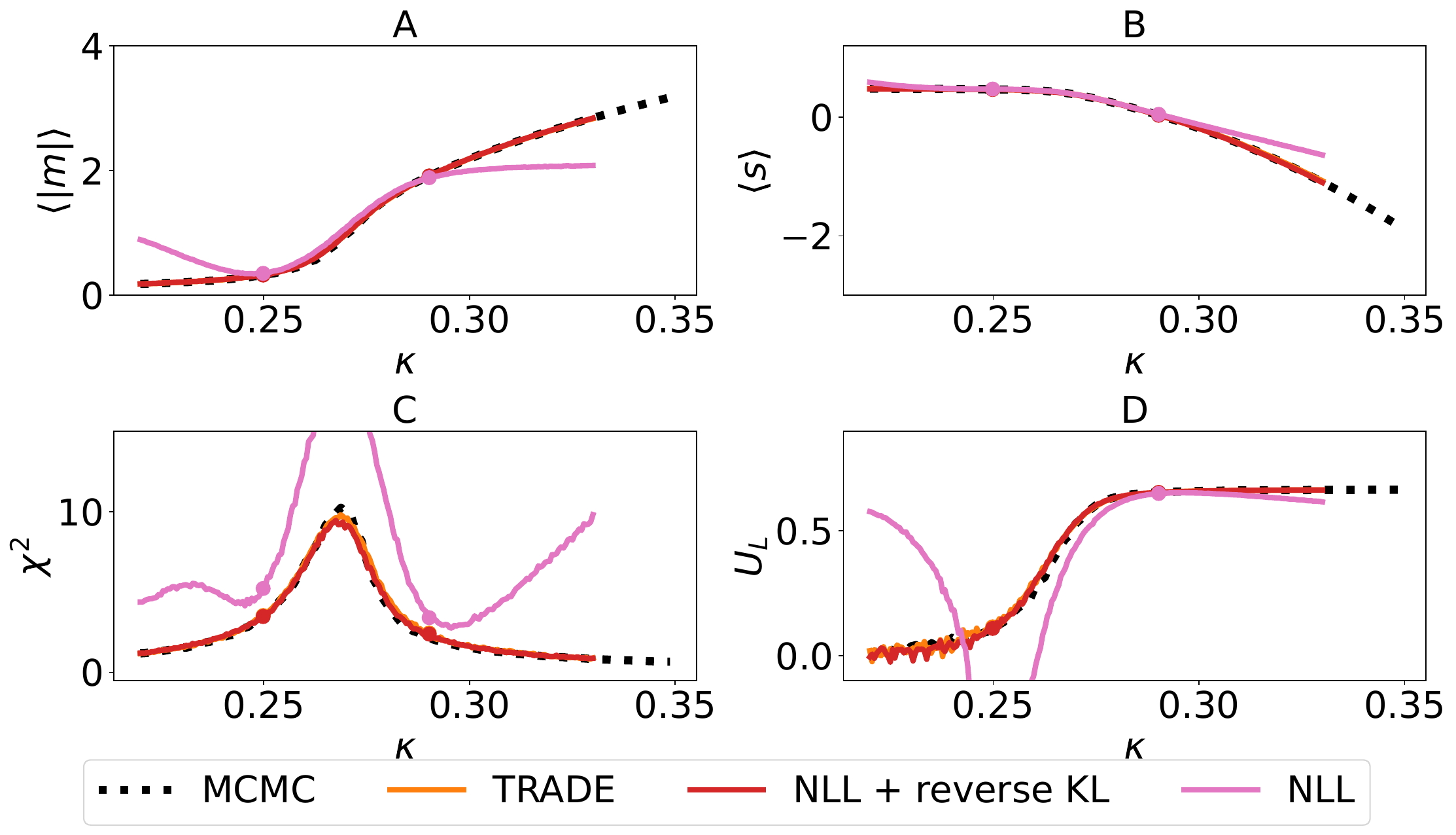}
    \caption{Physical observables for the scalar field theory. TRADE and a combination  of NLL and energy-based training accurately follow the ground truth obtained using MCMC, while NLL alone is detrimental. \textbf{A:} Expected absolute magnetization per spin. \textbf{B:} Expected ground truth action per spin. \textbf{C:} Susceptibility. \textbf{D:} Binder cumulant. The values marked with dots represent the $\kappa$ at which training data is available.}
    \label{fig:ScalarTheoryPhysicalProperties}
\end{figure}

\begin{figure}
    \centering
    \includegraphics[width=0.925\linewidth]{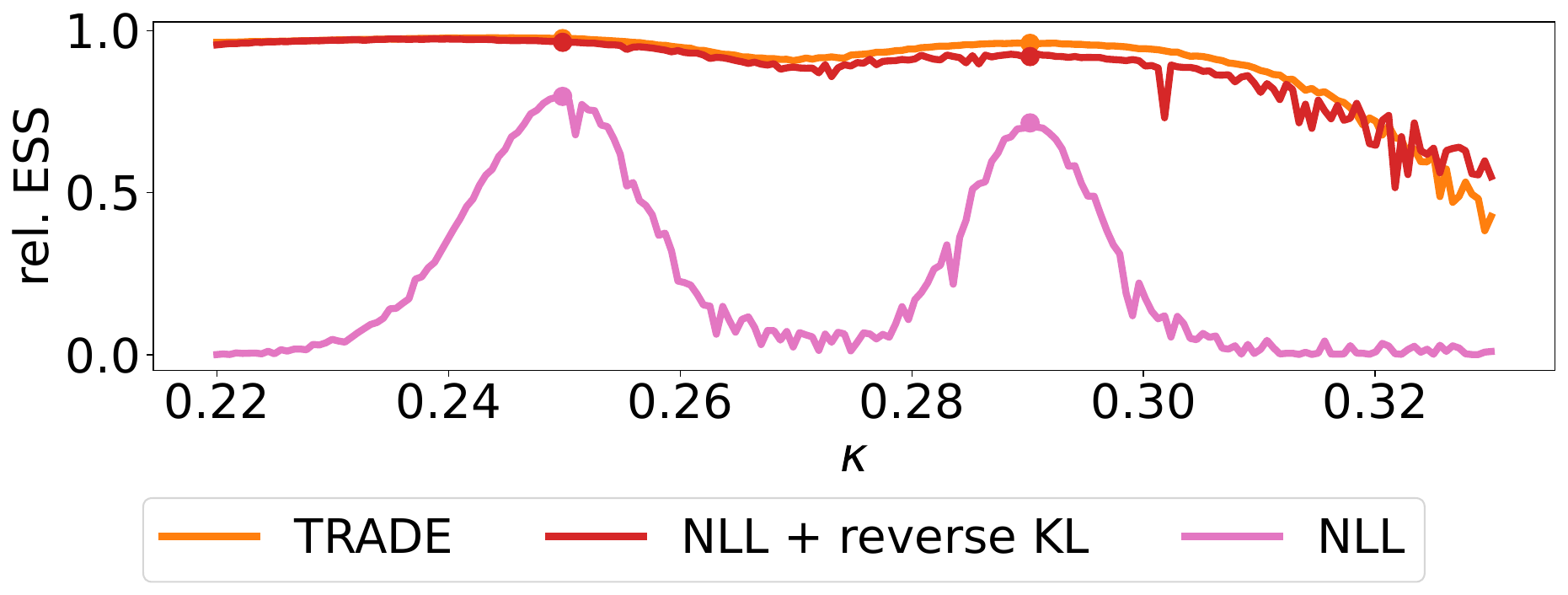}
    \caption{Relative ESS for the models trained for the scalar field theory. The values marked with dots represent the $\kappa$ at which training data is available. TRADE outperforms the baseline in most of the examined range of $\kappa$ and is also less fluctuating.}
    \label{fig:ScalarTheoryESS}
\end{figure}

In this section we apply TRADE to a physical lattice model which is for example examined by \citet{Pawlowski_2017}. The distribution of states depends on two external parameters, the quartic coupling $\lambda$ and the hopping parameter $\kappa$. The distribution of states is given by $p(\phi|\kappa,\lambda) \propto e^{-S(\phi,\lambda,\kappa)}$:
\begin{equation}\label{eq:ActionScalarTheory}
    \begin{matrix*}[l]
    S(\phi,\lambda,\kappa) =\\ \sum_x \left[- 2 \kappa \sum_{\mu = 1}^d\phi_x\phi_{x+\hat\mu}+ (1 - 2\lambda) \phi_x^2 + \lambda \phi_x^4\right].
    \end{matrix*}
\end{equation}
The first sum runs over all positions on the grid, while the second sum runs over all dimensions of the grid accounting for nearest-neighbor interactions. An interesting property of this model is that it has a second order phase transition as $\kappa$ is varied (see \citet{Pawlowski_2017}).

We fix $\lambda = 0.02$ and train a normalizing flow conditioned on $\kappa$ for $d = 2$ and a grid of size $8\times 8$. We use training data left and right of the phase transition for maximum likelihood training and interpolate the distribution between these two $\kappa$ by applying TRADE (in contrast to the previous experiments, where we only used data from one condition). As TSFs \citep{dibak2022temperature} are not applicable to this problem due to the structure of \cref{eq:ActionScalarTheory}, we use combined NLL and backward KL training as baseline. In addition, we train a model by only applying NLL training at the two $\kappa$ where data is available.

\Cref{fig:ScalarTheoryPhysicalProperties} shows some of the physical observables described in \citet{Pawlowski_2017} based on the samples following the three training schemes. The model trained with NLL only fails to reproduce the MCMC simulation. TRADE and the baseline using NLL and backward KL both accurately follow the ground truth data obtained via MCMC. \Cref{fig:ScalarTheoryESS} shows the corresponding ESS. Again, the NLL-only model falls short of the other models. TRADE and the NLL + reverse KL model perform similar at small $\kappa$, with a slight advantage for TRADE for intermediate values and for the baseline model at the highest values. Overall, the fluctuation in the ESS is smaller for TRADE than for the baseline model. Together, TRADE yields a performance gain compared to NLL training only and is competitive to the baseline. In \cref{app:ExperimentalDetails} we evaluate the stability of the training schemes.

\section{LIMITATIONS}
\label{sec:Limitations}

While TRADE yields excellent results in all our evaluations, there are several points we hope can be improved in the future. Firstly, the evaluation of the expectation value in \cref{eq:GradientBasedLossContributionTrade} may cause problems for high dimensional data sets or certain applications. Such problems may be extremely large batches, or extreme importance weights due to a bad alignment between the target and the proposal distribution. Secondly, TRADE requires an additional automatic differentiation step in each training iteration compared to other methods such as combined NLL and backward KL training, slightly increasing the computational requirements. However, none of these potential limitations posed issues in our experiments. 

\section{CONCLUSION}

In this work we introduce TRADE, a novel method that learns to transfer distributions between different external parameters. TRADE first establishes a boundary distribution at a fixed external parameter $c_0$ by learning e.g.~from i.i.d.~samples or energy-based training. It then propagates this information to other external parameters using the gradient of the unnormalized density w.r.t.~the external parameter. Previous methods in this field either had to rely on energy-based training, prone to mode collapse, or severely restrict model architecture. By leveraging the information gained from the functional form of the unnormalized density, TRADE avoids these limitations.

As a consequence, TRADE achieves excellent results in a diverse set of experiments, outperforming previous methods. We provide an extensive ablation study over design choices and benchmark against previous methods on a multidimensional well system. Our model can even be applied without access to the ground-truth unnormalized density, as demonstrated in a Bayesian inference setting. Finally, we demonstrate TRADE's ability to perform distribution transfer in complex settings such as learning configurations of Alanine Dipeptide and a scalar field theory lattice model.

\acknowledgments{This work is supported by Deutsche Forschungsgemeinschaft (DFG, German Research Foundation) under Germany's Excellence Strategy EXC-2181/1 - 390900948 (the Heidelberg STRUCTURES Cluster of Excellence). AR acknowledges funding from the Carl-Zeiss-Stiftung and by the German Federal Ministery of Education and Research (BMBF) (project EMUNE/031L0293A). HS acknowledges financial support by the German Research Foundation (DFG) through the Research Training Group 2450 “Tailored Scale-Bridging Approaches to Computational Nanoscience”. UK thanks the Klaus Tschira Stiftung for their support via the SIMPLAIX project. The authors acknowledge support by the state of Baden-Württemberg through bwHPC and the German Research Foundation (DFG) through grant INST 35/1597-1 FUGG.}

\newpage
% \subsubsection*{References}
\bibliography{references}
\bibliographystyle{apalike}

%%%%%%%%%%%%%%%%%%%%%%%%%%%%%%%%%%%%%%%%%%%%%%%%%%%%%%%%%%%%
\section*{Checklist}

 \begin{enumerate}
 
 \item For all models and algorithms presented, check if you include:
 \begin{enumerate}
   \item A clear description of the mathematical setting, assumptions, algorithm, and/or model. Yes, we provide this in \cref{sec:Method}.
   \item An analysis of the properties and complexity (time, space, sample size) of any algorithm. Yes, we provide a comparison in \cref{sec:Limitations}.
   \item (Optional) Anonymized source code, with specification of all dependencies, including external libraries. No, but we will publish code upon acceptance.
 \end{enumerate}

 \item For any theoretical claim, check if you include:
 \begin{enumerate}
   \item Statements of the full set of assumptions of all theoretical results. Yes, in \cref{sec:Method}.
   \item Complete proofs of all theoretical results. Yes, in \cref{sec:Method} and \cref{app:Proofs}.
   \item Clear explanations of any assumptions. Yes.
 \end{enumerate}

 \item For all figures and tables that present empirical results, check if you include:
 \begin{enumerate}
   \item The code, data, and instructions needed to reproduce the main experimental results (either in the supplemental material or as a URL). Yes, we list hyperparameters, used datasets and settings in appendix \cref{app:ExperimentalDetails} and will provide code upon acceptance.
   \item All the training details (e.g., data splits, hyperparameters, how they were chosen). Yes, in \cref{app:ExperimentalDetails}
         \item A clear definition of the specific measure or statistics and error bars (e.g., with respect to the random seed after running experiments multiple times). Yes, a description is provided with each experimental section and \cref{app:ExperimentalDetails}.
         \item A description of the computing infrastructure used. (e.g., type of GPUs, internal cluster, or cloud provider). Yes, in \cref{app:ExperimentalDetails}.
 \end{enumerate}

 \item If you are using existing assets (e.g., code, data, models) or curating/releasing new assets, check if you include:
 \begin{enumerate}
   \item Citations of the creator If your work uses existing assets. Yes, datasets are correctly attributed.
   \item The license information of the assets, if applicable. Yes, licenses are provided at the dataset sources.
   \item New assets either in the supplemental material or as a URL, if applicable. Not Applicable.
   \item Information about consent from data providers/curators. Not Applicable, all dataset used are publicly available.
   \item Discussion of sensible content if applicable, e.g., personally identifiable information or offensive content. Not Applicable.
 \end{enumerate}

 \item If you used crowdsourcing or conducted research with human subjects, check if you include:
 \begin{enumerate}
   \item The full text of instructions given to participants and screenshots. Not Applicable.
   \item Descriptions of potential participant risks, with links to Institutional Review Board (IRB) approvals if applicable. Not Applicable.
   \item The estimated hourly wage paid to participants and the total amount spent on participant compensation. Not Applicable.
 \end{enumerate}

 \end{enumerate}

\onecolumn
\appendix

\aistatstitle{Supplementary Materials}

\section{PROOFS}
\label{app:Proofs}
\subsection{Uniqueness of the Solution of the PDE used in the TRADE Objective}
\label{app:UniquenessProof}

In \cref{sec:ConditionalDistributionAsPDE} we describe the PDE on which TRADE is based, as outlined in \cref{eq:BoundaryDensity,eq:ConditionGradient}. To ensure that TRADE converges to the true solution, i.e. the ground truth $p(x|c)$ it remains to be shown that the boundary value problem \cref{eq:BoundaryDensity,eq:ConditionGradient} has a unique solution and that this solution corresponds to the ground truth. We begin by rewriting \cref{eq:BoundaryDensity,eq:ConditionGradient} in terms of the logarithm and show that this boundary value problem has a unique solution which is given by $\log p(x|c)$. The motivation behind switching to the logarithmic form is that our practical implementation is based on this version, as it is easier to implement and provides more numerical stability.

\begin{theorem}\label{thm:UniqueSolutionPDE}
    Let $p(x|c)$ be a conditional probability density with a continuous condition $c$ which is differentiable with respect to $c$. Then, the initial value problem 

    \begin{align}
        \log p_{\theta}(x|c_0) = \log p(x|c_0)\;\;\;\forall x\label{eq:Appendix-Boundary-Condition}\\
        \frac{\partial \log p_{\theta}(x|c)}{\partial c} = \frac{\partial \log p(x|c)}{\partial c}\;\;\; \forall x,c\label{eq:Appendix-Gradient-Condition}
    \end{align}

    has a unique solution $\log p_{\theta}(x|c)$, which is equal to $\log p(x|c)$.
\end{theorem}

\begin{proof}

We assume that we have two solutions $\log p_\theta(x|c)$ and $\log \tilde p_\theta(x|c)$ of the boundary value problem \cref{eq:Appendix-Boundary-Condition,eq:Appendix-Gradient-Condition}. Therefore, it follows that

\begin{equation}\label{eq:Appendix-Two-Solutions-PDE-Boundary-Condition}
    \log p_\theta(x|c_0) = \log \tilde p_\theta(x|c_0) = \log p(x|c_0)\;\;\;\forall x
\end{equation}

\begin{equation}\label{eq:Appendix-Two-Solutions-PDE-Gradient-Condition}
    \frac{\partial  \log p_\theta(x|c)}{\partial c} = \frac{\partial  \log \tilde p_\theta(x|c)}{\partial c} = \frac{\partial \log p(x|c)}{\partial c}\;\;\;\forall x,c.
\end{equation}

We are now interested in integrating \cref{eq:Appendix-Two-Solutions-PDE-Gradient-Condition} with respect to $c$. Without loss of generality, we consider the interval $[c_0,c]$. For $\log p_{\theta}(x|c)$ we find 

\begin{align}
    \log p_{\theta}(x|c) &= \log p_{\theta}(x|c_0) + \int_{c_0}^c \frac{\partial \log p(x|c')}{\partial c} dc' \label{eq:AppendixUniversalSolutionReweritten1}\\
    &\overset{(*)}{=} \log p(x|c_0) + \int_{c_0}^c \frac{\partial \log p(x|c')}{\partial c} dc'. \label{eq:AppendixUniversalSolutionReweritten2}
\end{align}

In $(*)$, we used the fact that $\log p_{\theta}(x|c)$ satisfies the boundary condition \cref{eq:Appendix-Boundary-Condition}. The right hand side of \cref{eq:AppendixUniversalSolutionReweritten2} only depends on the ground truth $\log p(x|c)$, and no longer on $\log p_\theta(x|c)$. Computing $\log \tilde p_{\theta}(x|c)$ analogously to \cref{eq:AppendixUniversalSolutionReweritten1,eq:AppendixUniversalSolutionReweritten1} we define the difference function \cref{eq:AppendixUniversalSolutionDifferenceFunction}.

\begin{equation}
        f(x,c) = \log p_{\theta}(x|c) - \log \tilde p_{\theta}(x|c)\label{eq:AppendixUniversalSolutionDifferenceFunction}
\end{equation}
\newpage

By inserting \cref{eq:AppendixUniversalSolutionReweritten2} and the corresponding result for $\log \tilde p_\theta(x|c)$ into \cref{eq:AppendixUniversalSolutionDifferenceFunction} we find

\begin{align}
    f(x,c) &= \log p(x|c_0) + \int_{c_0}^c \frac{\partial \log p(x|c')}{\partial c} dc' - \log p(x|c_0) - \int_{c_0}^c \frac{\partial \log p(x|c')}{\partial c} dc'\\
    &= 0.\label{eq:Appendix-Difference-function-is-zero}
\end{align}

\Cref{eq:Appendix-Difference-function-is-zero} directly shows that $\log p_\theta(x|c) \equiv \log \tilde p_\theta(x|c)$. Therefore, the boundary value problem \cref{eq:Appendix-Boundary-Condition,eq:Appendix-Gradient-Condition} has a unique solution.

To complete the proof, it is left to show, that the ground truth $\log p(x|c)$ is a solution of \cref{eq:Appendix-Boundary-Condition,eq:Appendix-Gradient-Condition}. This follows trivially from the definition of the boundary value problem. Combining this with the uniqueness of the solution, we conclude that \cref{eq:Appendix-Boundary-Condition,eq:Appendix-Gradient-Condition} have a unique solution, which is indeed the ground truth $\log p(x|c)$.
\end{proof}

\Cref{thm:UniqueSolutionPDE} proves that the only solution to \cref{eq:Appendix-Boundary-Condition,eq:Appendix-Gradient-Condition} is the logarithm of ground truth $p(x|c)$. However, \cref{eq:BoundaryDensity,eq:ConditionGradient} are formulated directly in terms of the distributions, rather than their logarithms. Nonetheless, the two formulations of the boundary value problem are equivalent, as demonstrated by the following theorem:

\begin{theorem}\label{thm:EquilvalenceBothPDeFormulations}

Let $\log p_{\theta}(x|c)$ be a solution of the boundary value problem \cref{eq:Appendix-Boundary-Condition,eq:Appendix-Gradient-Condition}. Then, $p_{\theta}(x|c) \coloneq e^{\log p_{\theta}(x|c)}$ is a solution of the boundary value problem

\begin{align}
    p_{\theta}(x|c_0) &= p(x|c_0)\;\;\;\forall x\label{eq:AppendixBVPFull1}\\
    \frac{\partial p_{\theta}(x|c) }{\partial c} &= \frac{\partial p(x|c) }{\partial c}\;\;\;\forall x,c.\label{eq:AppendixBVPFull2}
\end{align}

as introduced in \cref{eq:BoundaryDensity,eq:ConditionGradient} in \cref{sec:ConditionalDistributionAsPDE} of the main text.
\end{theorem}

\begin{proof}

    To prove \cref{thm:EquilvalenceBothPDeFormulations}, we insert $p_{\theta}(x|c) \coloneq e^{\log p_{\theta}(x|c)}$ into \cref{eq:AppendixBVPFull1,eq:AppendixBVPFull2}. Starting with \cref{eq:AppendixBVPFull1} we find:

    \begin{align}
        p_{\theta}(x|c_0) &= e^{\log p_{\theta}(x|c)}\\
        &\overset{(*)}{=}e^{\log p(x|c)}\\
        &= p(x|c_0).
    \end{align}

    For \cref{eq:AppendixBVPFull2} we compute:

    \begin{align}
        \frac{\partial p_\theta(x|c)}{\partial c} &= \frac{\partial e^{\log p_\theta(x|c)}}{\partial c}\\
        &= e^{\log p_\theta(x|c)} \frac{\partial \log p_\theta(x|c)}{\partial c}\\
        &\overset{(*)}{=} e^{\log p(x|c)} \frac{\partial \log p(x|c)}{\partial c}\\
        &= \frac{\partial e^{\log p(x|c)}}{\partial c}\\
        &= \frac{\partial p(x|c)}{\partial c}
    \end{align}

    In $(*)$, we use the fact that $\log p_{\theta}(x|c)$ is a solution of the boundary value problem \cref{eq:Appendix-Boundary-Condition,eq:Appendix-Gradient-Condition}, which implies $\log p_{\theta}(x|c) \equiv \log p(x|c)$.
\end{proof}

\begin{rem}
    The uniqueness of the solution of the boundary value problem \cref{eq:AppendixBVPFull1,eq:AppendixBVPFull2} can be shown analogously to the proof of  \cref{thm:UniqueSolutionPDE}. By combining (a)~the fact that $\log p(x|c)$ is the unique solution of \cref{eq:Appendix-Boundary-Condition,eq:Appendix-Gradient-Condition} and (b)~the result of \cref{thm:EquilvalenceBothPDeFormulations} with this uniqueness of the solution of \cref{eq:AppendixBVPFull1,eq:AppendixBVPFull2}, we conclude that $p(x|c)$ is the unique solution to \cref{eq:AppendixBVPFull1,eq:AppendixBVPFull2}.
\end{rem}

\subsection{Interchanging Derivative and Integration}\label{app:SwapDerivativeIntegral}

\Cref{thm:ConditionGradientUnormalized} in the main text is crucial to TRADE, as it allows for computing the derivative of the logarithm of the normalized density with respect to the condition in terms of the known unnormalized density. In the proof of \cref{thm:ConditionGradientUnormalized}, the derivative with respect to the condition is interchanged with the integral over the support of the density. In this section, we provide a theoretical justification for this interchange. 

\begin{theorem}\label{thm:Interchange-Of-Integral-And-Derivative}

Let $f(x,c)$ be a integrable function with a continuous condition $c$. We assume that $f(x,c)$ is differential with respect to $c$. For convenience, we define $I(c) \coloneq \int f(x,c) dx$. Furthermore, we assume that $\frac{\partial I(c)}{\partial c}$ exists. If one can construct a function $g(x,c)$ such that $|g_n(x,c)| \leq g(x,c)\;\forall x$ with $g_n(x,c) \coloneq \frac{f(x,c + h_n) - f(x,c)}{h_n}$, $h_n\rightarrow 0$ as $n\rightarrow\infty$, then 

\begin{equation}
    \frac{\partial }{\partial c} I(c) = \int  \frac{\partial }{\partial c} f(x,c) dx
\end{equation}

holds true.
\end{theorem}

Before we prove \cref{thm:Interchange-Of-Integral-And-Derivative}, we first cite a result concerning the interchange of limits and integrals:

\begin{theorem}[The Lebesgue Dominated Convergence Theorem, Proposition 6 of \cite{7TheIntegralofUnboundedFunctions})]\label{thm:LebesgueDominatedConvergenceTheorem}
If $\{f_n\}$ is a sequence of measurable functions and $f_n\rightarrow f$ a.e. on a set $S$ and there is an integrable function $g$ on $S$ such that $\left|f_n\right|\leq g$ for all $n$, then $f$ is integrable and $\int_S f_n\rightarrow \int_Sf$.
\end{theorem}

We are now ready to prove \cref{thm:Interchange-Of-Integral-And-Derivative}:

\begin{proof}
    Since $\frac{\partial I(c)}{\partial c}$ is assumed to exist, we can rewrite it using its definition:

    \begin{equation}
        \frac{\partial I(c)}{\partial c} = \lim_{n\rightarrow\infty} \frac{I(c + h_n) - I(c)}{h_n}
    \end{equation}

    Here, $\{h_n\}_{n\in\mathbb{N}}$ is a sequence in $\mathbb{R}$ with $h_n\rightarrow 0$ as $n\rightarrow\infty$. By exploiting the linearity of the integral, we find

    \begin{align}
        \frac{\partial I(c)}{\partial c} &= \lim_{n\rightarrow\infty} \frac{I(c + h_n) - I(c)}{h_n}\\
        &= \lim_{n\rightarrow\infty} \int \frac{f(x,c + h_n) - f(x,c)}{h_n} dx\\
        &= \lim_{n\rightarrow\infty} \int g_n(x,c)dx
    \end{align}

    where $g_n(x,c) \coloneq \frac{f(x,c + h_n) - f(x,c)}{h_n}$.
    
    Since $\frac{\partial f(x,c)}{\partial c}$ is assumed to exist, by the definition of the derivative, we know that $g_n(x,c)\rightarrow \frac{\partial f(x,c)}{\partial c}$ as $n \rightarrow \infty$. If we can construct a function $g(x,c)$ such that $|g_n(x,c)| \leq g(x,c)\;\forall x$, then applying \cref{thm:LebesgueDominatedConvergenceTheorem} concludes the proof.
\end{proof}

The proof of \cref{thm:Interchange-Of-Integral-And-Derivative} relies on the existence of a function $g$ that bounds $|g_n(x,\beta)|$. In \cref{example:Upper-Bound-Temperature-Scaling}, we construct such an upper bound for the case of temperature scaling, as for applied in \cref{sec:Ablation,sec:ExperimentAlanineDipeptide} of the main text.

\begin{example}\label{example:Upper-Bound-Temperature-Scaling}

In the case of temperature scaling, the function $f(x,c)$ used in \cref{thm:Interchange-Of-Integral-And-Derivative} can be identified with the unnormalized Boltzmann distribution $q(x|T) = e^{- \frac{E(x)}{k_BT}}$. For the sake of notation, we switch to the generalized inverse temperature $\beta \coloneq \frac{1}{k_B T}$, leading to $q(x|\beta) = e^{- \beta E(x)}$. We assume that the expectation value of the energy $\mathbb{E}_{p(x|\beta)}\left[E(x)\right]$ exists. This is a reasonable assumption, as the expected energy is a common observable in many applications. We can now define $g_n(x,\beta)$ as follows:

\begin{align}
    g_n(x,\beta) &= \frac{e^{-(\beta + h_n)\cdot E(x)} - e^{-\beta\cdot E(x)}}{h_n}\\
    &= e^{-\beta\cdot E(x)} \frac{e^{-h_n\cdot E(x)} - 1}{h_n}.
\end{align}

Let's assume without loss of generality, that $E(x)>0$ and $h_n > 0$. Under this assumption, we find 

\begin{equation}
    |g_n(x,\beta)| = e^{-\beta\cdot E(x)} \frac{1 - e^{-h_n\cdot E(x)}}{h_n}.
\end{equation}

To show that $g(x,\beta) = E(x)\cdot e^{-\beta\cdot  E(x)}$ is an upper bound for $|g_n(x,\beta)|$, it is sufficient to prove that $\frac{1 - e^{-h_n\cdot E(x)}}{h_n} \leq E(x)$. The function $\frac{1 - e^{-h_n\cdot E(x)}}{h_n} \leq E(x)$ is strictly monotonically decreasing in $h_n$ for $h_n > 0$. Therefore, the global maximum has to be attained for $h_n \rightarrow 0$. Using L'H\^{o}spitals rule, we find \cref{eq:Justification-L-Hospital}.

\begin{align}
    \lim_{n\rightarrow\infty}\frac{1 - e^{-h_n\cdot E(x)}}{h_n} &= \lim_{n\rightarrow\infty}\frac{E(x)\cdot e^{-h_n\cdot E(x)}}{1}\\
    &= E(x)\label{eq:Justification-L-Hospital}
\end{align}

In conclusion, for temperature scaling, $g(x,\beta) = E(x)\cdot e^{-\beta E(x)}$  serves as a valid upper bound for $|g_n(x,\beta)|$ used in the proof of \cref{thm:Interchange-Of-Integral-And-Derivative}.

\end{example}
\section{EXPERIMENTAL DETAILS}
\label{app:ExperimentalDetails}

\subsection{Loss Balancing}
\label{app:LossBalancing}

In the full TRADE objective in \cref{eq:FullLoss}, the two loss contributions $\mathcal{L}_{\text{boundary}}$ and $\mathcal{L}_{\text{grad}}$, are balanced by a hyperparameter $\lambda$. To ensure that both contributions are minimized equally during each update step, we do not use a fixed $\lambda$ throughout training. Instead, $\lambda$  is computed adaptively. This computation scheme is based on the loss balancing scheme described in \cite{wang2023expertsguidetrainingphysicsinformed}.

\begin{align}
    \hat \lambda_{\text{boundary}} &= \frac{||\nabla_\theta\mathcal{L}_{\text{boundary}}(\theta)|| + ||\nabla_\theta\mathcal{L}_{\text{grad}}(\theta)||}{||\nabla_\theta\mathcal{L}_{\text{boundary}}(\theta)||}\\
    \hat \lambda_{\text{grad}} &= \frac{||\nabla_\theta\mathcal{L}_{\text{boundary}}(\theta)|| + ||\nabla_\theta\mathcal{L}_{\text{grad}}(\theta)||}{||\nabla_\theta\mathcal{L}_{\text{grad}}(\theta)||}
\end{align}

To obtain $\lambda$ used in \cref{eq:FullLoss} we define $\lambda \coloneq \frac{\hat \lambda_{\text{grad}}}{\hat \lambda_{\text{boundary}}}$. By applying this weighting scheme, the magnitudes of the gradients of the two (weighted) contributions with respect to the model parameter $\theta$ are equal in each iteration. This ensures that the optimization of the objective is not dominated by a single contribution, but instead considers the full objective.

In practice, following \cite{wang2023expertsguidetrainingphysicsinformed},we update $\hat \lambda_{\text{boundary}}$ and $\hat \lambda_{\text{grad}}$ every $n$ training iterations. This reduces the computational overhead associated with the additional automated differentiation steps required to compute $\hat \lambda_{\text{boundary}}$ and $\hat \lambda_{\text{grad}}$. Additionally, we smooth $\hat \lambda_{\text{boundary}}$ and $\hat \lambda_{\text{grad}}$ by applying exponential averaging controlled by the parameter $\alpha_{\text{balance}}$. 

In our implementation, $\hat \lambda_{\text{boundary}}$ is computed based on the batch of training data used in the update step before the weights are updated. $\hat \lambda_{\text{grad}}$ is computed through a separate, data free calculation of $\mathcal{L}_{\text{grad}}$, using the current state of the training routine (i.e. causality weights and stored expectation values, as described in \cref{sec:practical-implementation}). 

\subsection{Sampling of the Condition for the Continuous Gradient Loss}
\label{app:ContinuousCSampling}
As described in \cref{sec:ConditionSampling}, we can define $p_{\operatorname{grad}}(c)$ based on causality weights only if we discretize the domain of $C$. In the alternative case, we define different sampling schedule from $p_{\operatorname{grad}}(c)$, which starts with the distribution concentrated around $c_0$ and progressively moves outward based on the current training step $t$. Given hyperparameters $s_\text{min}, s_\text{max}$, which determine the skew of the distribution towards $c_0$ at the start and end, respectively, we sample $c$ as follows:
\begin{align}
    & r \sim \mathcal{U}_{[ 0, 1 ]}  \\
    & w \sim \mathcal{B} \left( \frac{\log c_\text{max} - \log c_0}{\log c_\text{max} - \log c_\text{min}} \right) \\
    & \zeta := \exp \left( \log s_\text{min} + \frac{t}{t_\text{max}} \left( \log s_\text{max} - \log s_\text{min} \right) \right) \\
    & c = \begin{cases}
    \exp \left(\log c_0 + (1 - r ^ \zeta) (\log c_\text{max} - \log c_0) \right)& \text{if } w = 1 \\
    \exp \left(\log c_0 + (1 - r ^ \zeta) (\log c_\text{min} - \log c_0) \right)& \text{if } w = 0
    \end{cases}
\end{align}
The parameter $w$ chooses whether to move upward or downward from $c_0$ and $r^\zeta$ controls the magnitude of the shift. The decay/growth behavior of $r^\zeta$ is governed by $s_\text{min}$ and $s_\text{max}$. For $\zeta = 0$ we recover a $\delta$-distribution around $c = c_0$; $\zeta = 1$ results in a log-uniform distribution and $\zeta \rightarrow \infty$ results in $\delta$-distributions around $c_\text{min}$ and $c_\text{max}$. In general, we find that $s_\text{min} = 0.01, s_\text{max} = 1.5$ are good choices.

\subsection{Ablation Study: Multidimensional Wells}
\label{app:AppendixExperimentalDetailsAblation}

We generate data according to the energy given in \cref{eq:MultiwellEnergy} using 32 MCMC chains, each initialized in one of the basins. A step size of 500 is used for $T=1.0$, creating 300,000 data points while a step size of 1000 is used for $T=0.5$ to create 30,000 data points. All models use RQ-spline coupling blocks \citep{durkan2019neuralsplineflows} without augmentation dimensions, except for TSF, which employs 25 volume-preserving affine coupling blocks with five augmentation dimensions. We find that increasing the learning rate of TSF to $1 \times 10^{-3}$ improves its performance. \cref{tab:AblationHparams} lists the remaining hyperparameters.

\begin{table}
    \caption{Hyperparameters for the multidimensional well experiment. Parameters exclusively used by TRADE are highlighted in gray.\newline}
    \centering
    \begin{tabular}{l|c}
        \toprule
        Hyperparameter & Value \\
        \midrule
        Train, val, test split & 0.8, 0.1, 0.1 \\
        Coupling blocks & 15 \\
        Coupling type & RQ Splines \citep{durkan2019neuralsplineflows} \\
        Hidden dimensions & [256, 256] \\
        Latent distribution & normal \\
        Activation & SiLU \\
        Learning rate & $2 \times 10^{-4}$ \\
        Weight decay & $0.0$ \\
        Optimizer & Adam \\
        Gradient clipping & 3.0 \\
        Dequantization noise & $0.0$ \\
        Batch size & 512 \\
        Training steps & 10,000 \\
        Lr scheduler & one-cycle lr \citep{smith2018superconvergencefasttrainingneural} \\
        \midrule
        \rowcolor{customgray} $\lambda_{\mathcal{L}_\text{grad}}$ & $1.0$ \\
        \rowcolor{customgray} Start $\mathcal{L}_\text{grad}$ at step & 2,000 \\
        \rowcolor{customgray} Grid points & 15 \\        
        \rowcolor{customgray} $\epsilon$ causality weights & $0.9$ \\
        \rowcolor{customgray} Expectation value update interval & 100 steps \\
        \rowcolor{customgray} Expectation value update samples & 5,000 \\
        \rowcolor{customgray} $p_{\operatorname{grad}}(x|c)$ & $p_\theta(x|c) + \mathcal{N}(0, 0.0003^2)$\\
        \bottomrule
    \end{tabular}
    \label{tab:AblationHparams}
\end{table}

\subsection{Tempered Bayesian Inference: Two Moons}
%\label{app:ExperimentalDetailsTwoMoons}

We train models on the task introduced by \cite{greenberg2019automatic}, but modify the prior distribution to $p(\psi) = \mathcal{N}(0, 0.3^2)$, as described in the main text. For training, we need to estimate 
\begin{equation}
    \frac{\partial \log p_\beta(\psi | y)}{\partial \beta} \propto \log p(y | \psi),
\end{equation}
which we approximate as $\log p(y | \psi) \approx \log p_\theta(\psi | y) - \log p(\psi)$. 

For evaluation, we need to sample from $\log p(\psi | y, \beta)$ in order to compute the calibration curves. This is equivalent to sampling from $p(\psi)$ and generating observations from $p(y | \psi)^\beta$. We can approximate $p(y | \psi)^\beta \propto p(r)^\beta p(\alpha)^\beta$, which is tractable since $p(r)^\beta = \mathcal{N}(0.1, \frac{0.01^2}{\beta})(r)$ and $p(\alpha)^\beta = p(\alpha)$. 

The likelihood $p(y | \psi)$ forms a half-moon distribution, with the width being controlled by $r$. As $\beta$ becomes too small, the variance of $r$ increases, and the approximation $p(y | \psi)^\beta \propto p(r)^\beta p(\alpha)^\beta$ may become inaccurate. However, in the considered regime of $\beta \in [0.5, 2.0]$, the approximation is sufficiently accurate. 

We compute the calibration by drawing 300,000 pairs $y, \psi$ as described above, then for each $y$, we draw 300 model samples of $\tilde{\psi}$ and compute the corresponding $\tilde{r}$. Starting from the median of the $\tilde{r}$, we move outward until we reach the ground truth $r$, and report the coverage obtained in this way in \cref{fig:TwoMoonsResults}. The 95\% confidence intervals are obtained via bootstrap sampling. 

Hyperparameters are listed in \cref{tab:TwoMoonsHparams} and in \cref{fig:TwoMoonsSupplementary}, we provide samples from the learned posterior distribution at different $\beta$ for $r, \alpha = 0$. 

\begin{table}
    \caption{Hyperparameters for the two moons experiment. Parameters exclusively used by TRADE are highlighted in gray.\newline}
    \centering
    \begin{tabular}{l|c}
        \toprule
        Hyperparameter & Value \\
        \midrule
        Train, val, test split & 0.8, 0.1, 0.1 \\
        Coupling blocks & 15 \\
        Coupling type & RQ Splines \citep{durkan2019neuralsplineflows} \\
        Hidden dimensions & [128, 128] \\
        Latent distribution & normal \\
        Activation & SiLU \\
        Learning rate & $4 \times 10^{-4}$ \\
        Weight decay & $0.0$ \\
        Optimizer & Adam \\
        Gradient clipping & 3.0 \\
        Dequantization noise & $3 \times 10^{-4}$ \\
        Batch size & 512 \\
        Training steps & 10,000 \\
        Lr scheduler & one-cycle lr \citep{smith2018superconvergencefasttrainingneural} \\
        \midrule
        \rowcolor{customgray} $\lambda_{\mathcal{L}_\text{grad}}$ & $0.1$ \\
        \rowcolor{customgray} Start $\mathcal{L}_\text{grad}$ at step & 2,000 \\
        \rowcolor{customgray} Discretize & no \\
        \rowcolor{customgray} $p_{\operatorname{base}}(x|c)$ & $p_\theta(x|c)$ \\
        \rowcolor{customgray} $p_{\operatorname{grad}}(x|c)$ & $p_\theta(x|c) + \mathcal{N}(0, 0.0003^2)$\\
        \bottomrule
    \end{tabular}
    \label{tab:TwoMoonsHparams}
\end{table}

\begin{figure}
    \centering
    \includegraphics[width=0.32\linewidth]{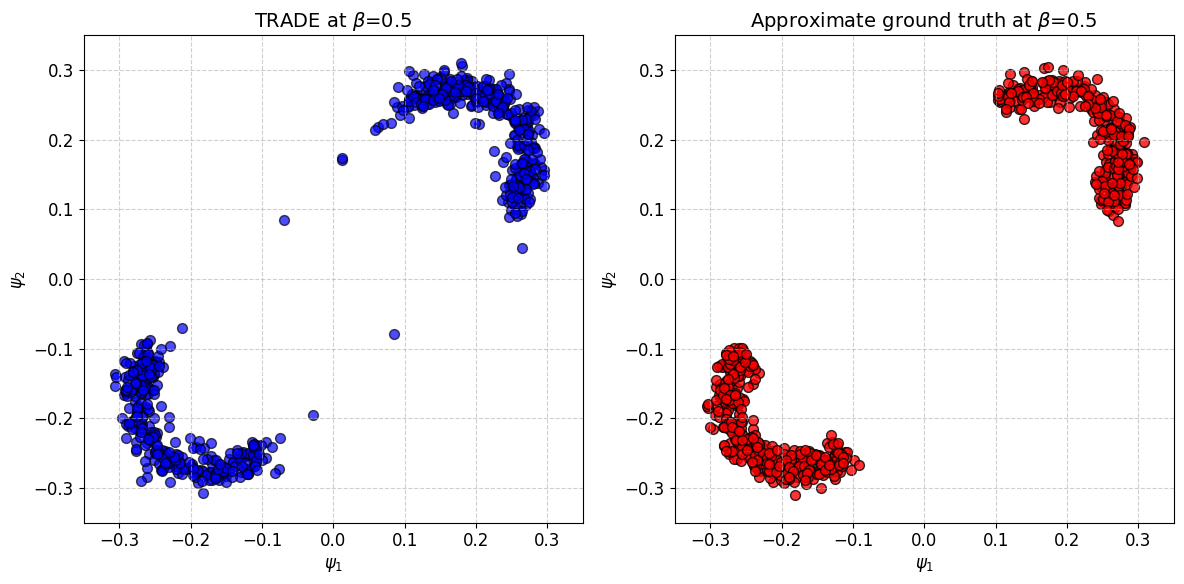}
    \vrule width 1pt height 2.8cm % Vertical line
    \includegraphics[width=0.32\linewidth]{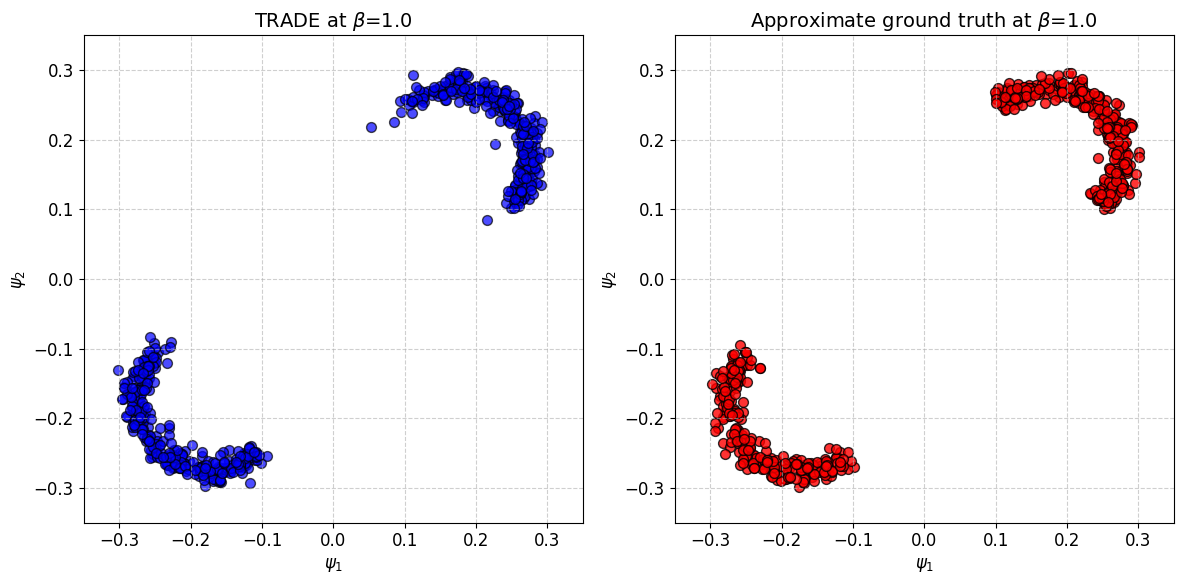}
    \vrule width 1pt height 2.8cm % Vertical line
    \includegraphics[width=0.32\linewidth]{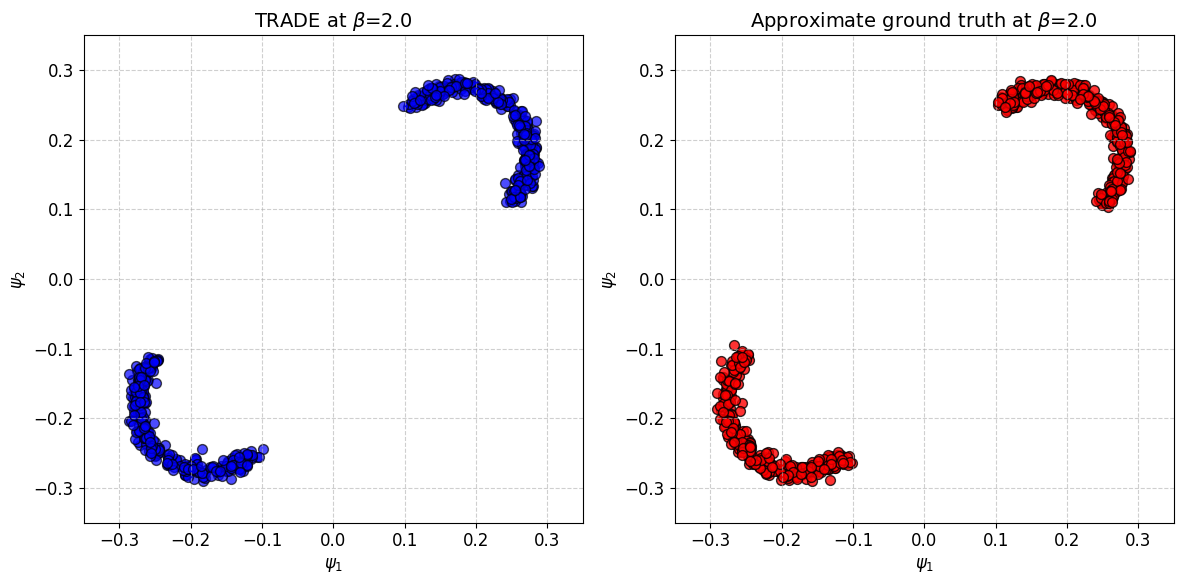}
    \caption{Learned (blue) vs. estimated (red) posterior samples of $\psi$ at different $\beta$ values for the observation $r, \alpha = 0$, by TRADE in the two moons experiment.}
    \label{fig:TwoMoonsSupplementary}
\end{figure}

\subsection{Temperature Scaling of Alanine Dipeptide}
%\label{app:ExperimentalDetailsAlanineDipeptide}

We follow the setting of \cite{dibak2022temperature} to train a normalizing flow that learns the distributions of Alanine Dipeptide configurations at $T \in [300K, 600K]$, using their publicly available dataset. The configurations are first transformed into internal coordinates (torsions, angles, bonds), using a similar coordinate transform to \cite{noe2019boltzmann}. All three methods are tested for affine coupling blocks \citep{dinh2016density} and rational quadratic spline coupling blocks \cite{durkan2019neuralsplineflows}. Affine coupling blocks are made volume-preserving for TSF by normalizing the logarithmic scaling coefficients $s$ of each block to 0. For a volume-preserving spline coupling architecture we use the architecture proposed in figure 1 of \citet{dibak2022temperature} in combination with the approximation for temperature steerable rational quadratic splines the authors present in appendix B. The latent distributions for affine models are scaled according to the temperature at which they are evaluated. The subnetworks of each block are single residual blocks, with an input and output layer scaling to and from the hidden size. The remaining hyperparameters are listed in \cref{tab:Ala2Hparams}. Additionally, we provide a visualizations of results of all architectures at $T = 600K$ and $T=300K$ in \cref{fig:AlanineDipeptideRemainingResults}. We performed a preliminary hyperparameter search using optuna \citep{optuna_2019}, then finetuned them for each model by hand.

\begin{table}
    \caption{Hyperparameters for TRADE and TSF for the temperature scaling of Alanine Dipeptide. Parameters exclusively used by TRADE are highlighted in gray. Whenever different methods use different hyperparameters they are listed in the order TRADE/TSF/Backward KL \newline}
    \centering
    \resizebox{\linewidth}{!}{
    \begin{tabular}{l|cc}
        \toprule
        Hyperparameter & Affine Models & Spline Models\\
        \midrule
        Train, val, test split & 0.8, 0.05, 0.15 & 0.8, 0.05, 0.15 \\
        Coupling blocks & 20 & 18\\
        Hidden dimensions & [128, 128, 128] & [256, 256] \\
        Augmentation dimensions \citep{huang2020augmented} & 60 & 0\\
        Latent distribution & normal & uniform + truncated normal \\
        Activation & SiLU & SiLU \\
        Learning rate & $3 \times 10^{-4}$ & $1 \times 10^{-4}$ \\
        Weight decay & $1 \times 10^{-6}$/$1 \times 10^{-5}$/$1 \times 10^{-5}$ & $3.7 \times 10^{-6}$ \\
        Optimizer & Adam & Adam \\
        Gradient clipping & 3.0 & 3.0 \\
        Dequantization noise & $3 \times 10^{-4}$ & $3 \times 10^{-4}$ \\
        Batch size & 1024 & 289 \\
        Training steps & 58k & 83k \\
        Lr scheduler & one-cycle lr \citep{smith2018superconvergencefasttrainingneural} & one-cycle lr  \\
        \midrule
        \rowcolor{customgray} $\lambda_{\mathcal{L}_\text{grad}}$ & 0.1 & 0.23\\
        \rowcolor{customgray} Start $\mathcal{L}_\text{grad}$ at step & 23k & 25k \\
        \rowcolor{customgray} Discretize & no & yes\\
        \rowcolor{customgray} Grid points & - & 25 \\        
        \rowcolor{customgray} $\epsilon$ causality weights & - & 0.12 \\
        \rowcolor{customgray} Expectation value update interval & - & 1,000 steps \\
        \rowcolor{customgray} Expectation value update samples & - & 5,000 \\
        \rowcolor{customgray} $p_{\operatorname{base}}(x|c)$ & $p_\theta(x|c)$ & $p_\theta(x|c)$ \\
        \rowcolor{customgray} $p_{\operatorname{grad}}(x|c)$ & $p_\theta(x|c) + \mathcal{N}(0, 0.0003^2)$ & $p_\theta(x|c) + \mathcal{N}(0, 0.0003^2)$\\
        \rowcolor{customgray} Use ground truth $q(x|c)$ & yes & yes \\
        \bottomrule
    \end{tabular}}
    \label{tab:Ala2Hparams}
\end{table}

\begin{figure}
    \centering
    \includegraphics[width=\linewidth]{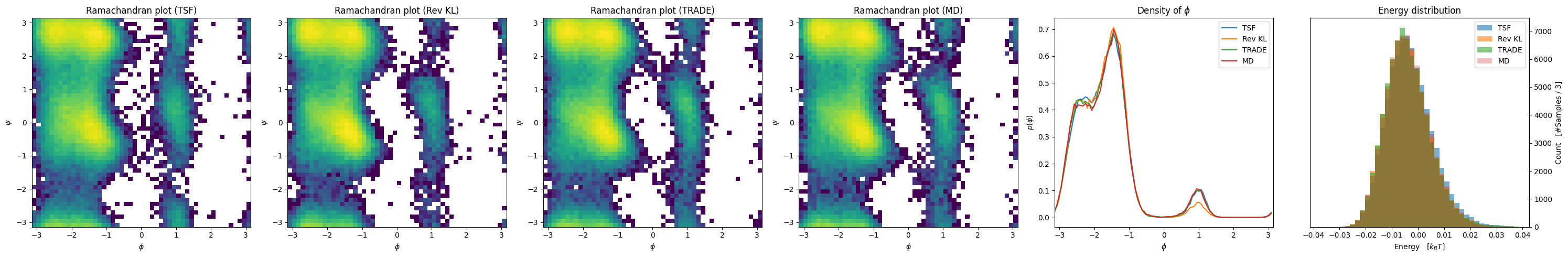}
    \includegraphics[width=\linewidth]{figures/ala2/ala2_300K_spline.png}
    \includegraphics[width=\linewidth]{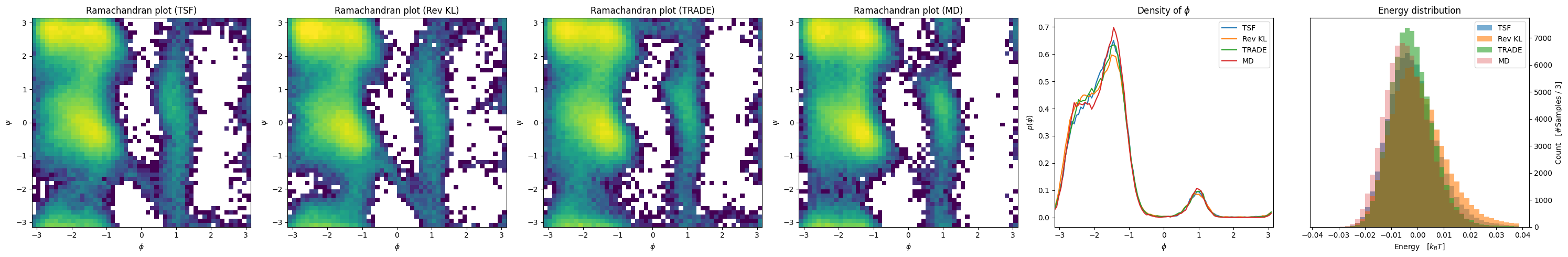}
    \includegraphics[width=\linewidth]{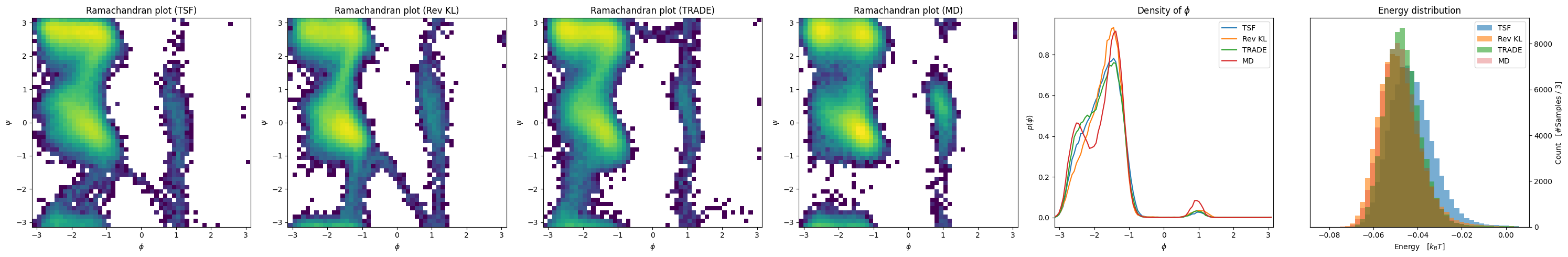}
    \caption{A comparison of temperature steerable flows (TSF, \cite{dibak2022temperature}) and backward KL with TRADE at different temperatures and architectures for Alanine Dipeptide. From left to right: Ramachandran plots of model samples (TSF, backward KL, TRADE) and molecular dynamics (MD), marginal density of the $\phi$ angle, and ground truth energy of model samples and MD samples (red). From top to bottom: Spline model $T=600K$, spline model $T=300K$, affine model $T=600K$, affine model $T=300K$}
    \label{fig:AlanineDipeptideRemainingResults}
\end{figure}

\subsection{Varying External Parameters in a Scalar Field Theory}
%\label{app:ExperimentalDetailsScalarTheory}

\subsubsection{Physical Observables}

In \cref{fig:ScalarTheoryPhysicalProperties} we visualize some physical observables taken from \cite{Pawlowski_2017}. The expected absolute magnetization per spin (subplot A) is defined by \cref{eq:Appendix-Absolute-Mean-magnetization-per-spin}. The expected ground truth action per spin (subplot B) is given by \cref{eq:Appendix-expected-action-per-spin}, where $S(\phi,\kappa)$ corresponds to \cref{eq:ActionScalarTheory} with $\lambda = 0.02$. \Cref{eq:Appendix-Susceptibility} defines the susceptibility (subplot C), and finally, \cref{eq:Appendix-Binder-Cumulant} defines the Binder cumulant (subplot D). Here, $N$ represents the number of lattice points.

\begin{align}
    m(\phi) &= \frac{1}{N} \left|\sum_x \phi_x \right|\\
    \left< m(\kappa)\right> &= \EX_{p_\theta(x|\kappa)}\left[m(\phi)\right]\label{eq:Appendix-Absolute-Mean-magnetization-per-spin}\\
    \left< s(\kappa) \right> &= \frac{1}{N} \EX_{p_\theta(x|\kappa)}\left[S(\phi,\kappa)\right]\label{eq:Appendix-expected-action-per-spin}\\
    \chi^2(\kappa) &= N\left(\EX_{p_\theta(x|\kappa)}\left[m(\phi)^2\right] - \EX_{p_\theta(x|\kappa)}\left[m(\phi)\right]^2\right)\label{eq:Appendix-Susceptibility}\\
    U_L(\kappa) &= 1 - \frac{1}{3}\frac{\EX_{p_\theta(x|\kappa)}\left[m(\phi)^4\right]}{\EX_{p_\theta(x|\kappa)}\left[m(\phi)^2\right]^2}\label{eq:Appendix-Binder-Cumulant}
\end{align}

In our experiments, the expectation values in \cref{eq:Appendix-Absolute-Mean-magnetization-per-spin} to \cref{eq:Appendix-Binder-Cumulant} are approximated by the averaging over 10,000 samples drawn from the learned distributions.

\subsubsection{Model Architecture and Training Details}

\textbf{Model}

The three models used in the experiments in \cref{sec:ExperimentScalarTheory} share same model architecture. Instead of directly using the external parameter as condition for the networks, we transform the condition by passing the logarithm of the external parameter through a small residual network. This network consists of one residual block containing a fully connected neural network with two hidden layers, each with 32 hidden neurons.  

The invertible function is structured as follows: In the first layer, we apply a channel-wise global linear transformation. The parameters of this transformation are computed by fully connected neural networks with two hidden layers and 64 hidden neurons. The input to these models is the transformed external parameter. After this global scaling, the input images, which have dimensionality $(1,H,W)$, are transformed to a shape of $(4,H / 2,W / 2)$ using an invertible down-sampling layer  \citep{jacobsen2018irevnetdeepinvertiblenetworks}. This layer is followed by three rational-quadratic neural spline coupling blocks \citep{durkan2019neuralsplineflows} which use convolutional subnetworks as described in \cref{tab:ConvSubnetWorkScalarTheory}. 

Next, the data is passed through an invertible flattening operator and three additional rational-quadratic neural spline coupling blocks, this time using fully connected subnetworks (see \cref{tab:FCSubnetWorkScalarTheory}). The transformed external parameter is passed to all six rational-quadratic neural spline coupling blocks via an additional channel added to the input of the coupling block. The latent distribution is a 64-dimensional standard normal distribution. All networks use SiLU activations. The weights of the subnetworks used in the coupling blocks are initialized using Xavier normal initialization \citep{pmlr-v9-glorot10a}. Additionally, the weights of the final layer in each subnetwork are initialized with zeros.

\begin{table}[ht]
    \caption{Convolutional subnetworks used in the rational-quadratic neural spline coupling blocks of the normalizing flows for the scalar theory experiments on a lattice, as described in \cref{sec:ExperimentScalarTheory}.\newline}
    \centering
    \begin{tabular}{c|c|c|c}
         Step&Operation&Input shape&Output shape\\
         \hline
         1&3x3 Convolution&$c_{in} \times H \times W$&$64 \times H \times W$\\
         2&SiLU Activation&$64 \times H \times W$&$64 \times H \times W$\\
         3&3x3 Convolution&$64 \times H \times W$&$64 \times H \times W$\\
         4&SiLU Activation&$64 \times H \times W$&$64 \times H \times W$\\
         5&1x1 Convolution&$64 \times H \times W$&$32 \times H \times W$\\
         6&SiLU Activation&$32 \times H \times W$&$32 \times H \times W$\\
         7&1x1 Convolution&$32 \times H \times W$&$32 \times H \times W$\\
         8&SiLU Activation&$32 \times H \times W$&$32 \times H \times W$\\
         9&1x1 Convolution&$32 \times H \times W$&$c_{out} \times H \times W$
    \end{tabular}
    \label{tab:ConvSubnetWorkScalarTheory}
\end{table}

\begin{table}[ht]
    \caption{Fully connected subnetworks used in the rational-quadratic neural spline coupling blocks of the normalizing flows for the scalar theory experiments on a lattice, as described in \cref{sec:ExperimentScalarTheory}.\newline}
    \centering
        \begin{tabular}{c|c|c|c}
         Step&Operation&Input shape&Output shape\\
         \hline
         1&Linear&$c_{in}$&$128$\\
         2&SiLU Activation&$128$&$128$\\
         3&Linear&$128$&$128$\\
         4&SiLU Activation&$128$&$128$\\
         5&Linear&$128$&$128$\\
         6&SiLU Activation&$128$&$128$\\
         7&Linear&$128$&$128$\\
         8&SiLU Activation&$128$&$128$\\
         9&Linear&$128$&$c_{out}$
    \end{tabular}
    \label{tab:FCSubnetWorkScalarTheory}
\end{table}

\textbf{Data}

Training data is generated using Langevin dynamic based on the drift term stated in \cite{Pawlowski_2017}. The MCMC is initialized from Gaussian noise, and we select the states that are in equilibrium. To ensure, that the samples are uncorrelated, we compute the correlation times between the samples and only keep samples that are two correlation times apart from each other.

To obtain training and validation data, we start two different MCMC runs, each with a different random seed.

To ensure that the validation data covers all modes of the data distribution, we use 1,000 samples from the MCMC and apply data augmentation by flipping the sign of the whole states, as well as flipping them horizontally and vertically, along with combinations of these transformations. This results in validation sets of size 8,000 for each validation $\kappa$.

For the NLL contribution to TRADE and the combined NLL and backward KL training, we randomly apply data augmentation to the selected batch of training data in each training iteration. This augmentation exploits the symmetries of the system under investigation. Specifically, we randomly flip the sign of the states, apply random mirroring, and random translations of the states. 

\textbf{Training}

The hyperparameters for the training with TRADE, NLL-only, and the combined NLL and backward KL objective are summarized in \cref{tab:ScalarTheoryTradeHparams}. The training of TRADE and NLL+backward KL follows a two-stage training: first, the model is only trained using only the NLL loss for $n_{\text{nll}}$ epochs. After that, the full objective is applied for the remaining $n_{\text{epochs}} - n_{\text{nll}}$ epochs. In the second phase, the learning rate is multiplied by a decay factor $\gamma_{\text{lr}}$. 

In both phases, a one-cycle learning rate scheduler \citep{smith2018superconvergencefasttrainingneural} is used. Loss balancing, as described in \cref{app:LossBalancing}, is applied to TRADE and the NLL + backward KL training. Additionally, we allow a fixed relative weighting factor $\lambda_{\text{fixed}}$, which is multiplied by the adaptive $\lambda$ obtained from the dynamic weighting scheme. 

For both training schemes, the range of $\kappa$ used in the computation of the data-free loss contribution is given by $I_{\kappa}$. For TRADE, this interval is discretized into $N_{\text{grid}}$ grid points.

%TRADE specific stuff
To encourage the model to focus on the region between the two $\kappa$ values where training data is available, $\epsilon$, as used in the computation of the importance weights (see \cref{sec:ConditionSampling}), is reduced in this region for TRADE by multiplying it by $\gamma_{\epsilon}$. Across the entire grid, $\epsilon$ decays exponentially over the course of training, reaching a final ratio of $r_\epsilon$ relative to its initial value. This is intended to broaden the distribution of sampled grid points as the training progresses. 

For $p_{\text{grad}}(x|c)$ (see \cref{sec:EvaluationPointSampling}) we use $p_{\theta}(x|\kappa)$. Additionally, we apply the same random data augmentation to the proposed samples as we do to the training data for the NLL contribution and add pixel-wise Gaussian noise with standard deviation $\sigma_{\text{noise}}$. To prevent training from being affected by outliers, the Huber loss with a parameter $\delta_{\text{Huber}}$ is used in the computation of $\mathcal{L}_{\text{grad}}$, instead of the mean squared error described in \cref{eq:GradientBasedLossContributionTrade}. 
The exponential averaging of the stored expectation values on the grid points is governed by $\alpha_{\text{expectation}}$ and the exponential averaging of $\mathcal{L}_{\text{grad}}$ at the grid points is governed by $\alpha_{\text{grad}}$.

For the combined NLL and backward KL training, samples from $p_{\theta}(x|\kappa)$ are required for the backward KL contribution. In the experiment, we sample the logarithm of the condition $\kappa$ uniformly.

For the training off the NLL-only model the exactly same parameters as for the model trained with TRADE are applied, except that the data-free loss contribution $\mathcal{L}_{\text{grad}}$ is not used.

\subsubsection{Model Evaluation}\label{app:ScalarTheoryEvaluation}

To evaluate the performance of a certain model, we compute the NLL of the validation sets at different values of $\kappa$ (See entry "Evaluation parameters" in \cref{tab:ScalarTheoryTradeHparams}) and then compute the average of the NLLs across these $\kappa$. This average serves as a measure of the model's overall performance. The models are evaluated every five epochs during training, and the model with the lowest average validation NLL is selected for further examinations.

\subsubsection{Additional Experimental Results}

In this section, we want to examine the stability of the TRADE training scheme. To do this, we use the training configuration described in \cref{tab:ScalarTheoryTradeHparams} and repeat the training seven times, each with a different random seed. The best model from each run is selected following the procedure outlined in \cref{app:ScalarTheoryEvaluation}. For every model obtained, we compute the relative ESS as a function of $\kappa$. Each ESS value is based on 10,000 samples drawn from the learned distribution. The same evaluation is repeated for the combined NLL and backward KL training.

In \cref{fig:Appendix-Scalar-Theory-ESS-different-random-seed}, we average the relative ESS over the seven runs for both training schemes (solid lines) and also visualize the interval defined by the standard deviation (shaded regions). For reference, we also plot the ESS for the one model trained with NLL-only on the $\kappa$ where training data is available.

It is clearly observable, that both training schemes perform best near the $\kappa$ values where NLL training is performed, with the smallest variation in these regions. However, across the full examined range of $\kappa$, TRADE shows significantly smaller fluctuation. In particular, for larger $\kappa$ values, the combined NLL and backward KL training exhibits a much larger standard deviation.

This evaluation indicates that TRADE produces more reliable training outcomes, showing far less sensitivity to changes in the random seed compared to the baseline of combined NLL and reverse KL training.

\Cref{fig:ScalarTheoryPhysicalProperties,fig:ScalarTheoryESS} in \cref{sec:ExperimentScalarTheory} are based on the best model (with respect to the average validation NNL) out of the seven runs for each training scheme.

\subsection{Software Packages}
\label{app:SoftwarePackages}

Our experiments are implemented in Python and rely on the following packages: For the implementation of the training routines, we use PyTorch \citep{paszke2019pytorch} and PyTorch Lightning \citep{Falcon_PyTorch_Lightning_2019}. Training progress is tracked by using TensorFlows's \citep{tensorflow2015-whitepaper} TensorBoard. Our invertible neural networks are based on FrEIA \citep{freia}. For general computations, we use NumPy \citep{harris2020array} and pandas \citep{mckinney-proc-scipy-2010,reback2020pandas}. Plotting is done with Matplotlib \citep{Hunter2007Matplotlib}. To systematically test different hyperparameter configurations, we use Optuna \citep{optuna_2019}.

\subsection{Compute Resources}
\label{app:ComputeResources}

To run our experiments, we used up to 12GB of GPU working memory. All our experiments can be completed in less than ten hours on such a device. In addition to our off-the-shelf desktop computers, we had access to a high-performance computing cluster, which we used for some of the experiments.

\begin{table}
    \caption{Hyperparameter for the models trained on the lattice field theory in \cref{sec:ExperimentScalarTheory}. The \textbf{first block} corresponds to parameters concerning the general training routine. The \textbf{second block} details the specific settings for the data-free loss contribution, and the \textbf{third block} describes the dataset and the validation procedure for the trained models. A "-" indicates that this parameter is not applicable the corresponding training scheme.\newline}
    \centering
    \begin{tabular}{l|c|c|c}
        \toprule
         & TRADE & NLL-only & NLL + backward KL \\
        \midrule
        Learning rate & $7.66 \times 10^{-5}$&$7.66 \times 10^{-5}$&$1.33 \times 10^{-3}$\\
        $\gamma_{\text{lr}}$ & 3.0&3.0&0.5\\
        $n_{\text{epochs}}$ & 400&400&400\\
        $n_{\text{nll}}$ & 20&20&73\\
        Weight decay & 0.0&0.0&0.000135\\
        Batch size (NLL)& 1024&1024&1024\\
        Gradient clipping& 1,000&1,000&16.0\\
        Optimizer & Adam&Adam&Adam\\
        \midrule
        $\lambda_{\text{fixed}}$& 0.355&-&3.847\\
        $\sigma_{\text{noise}}$& 0.38&-&-\\
        $I_{\kappa}$&$[0.22,0.32]$&-&$[0.22,0.32]$\\
        $\delta_{\text{Huber}}$&6.87&-&-\\
        Batch size (data-free loss) & 2048&-&2048\\
        $\epsilon$ causality weights & $1.57\times 10^{-6}$&-&-\\
        $\gamma_\epsilon$&0.376&-&-\\
        $r_\epsilon$&0.228&-&-\\
        $N_{\text{grid}}$ & 150&-&-\\
        $\alpha_{\text{expectation}}$&0.1&-&-\\
        $\alpha_{\text{grad}}$&0 $2.5\times 10^{-5}$&-&-\\
        $\alpha_{\text{balance}}$ &0.1&-&0.1\\
        Update frequency expectation value&2,000 iterations&-&-\\
        Samples expectation value update&1500&-&-\\
        \midrule
        Base parameters for NLL loss &  $\{0.25,0.29\}$&$\{0.25,0.29\}$&$\{0.25,0.29\}$\\
        Independent samples per base parameter & 100,000&100,000&100,000\\
        Samples per validation parameter & 8,000&8,000&8,000\\
        Evaluation parameters&$\{0.24,0.25,...,0.3\}$&$\{0.24,0.25,...,0.3\}$&$\{0.24,0.25,...,0.3\}$\\
        Random data augmentation (training)& yes&yes&yes\\
        Dequantization noise & no& no & no\\
        Lattice size &$8\times8$&$8\times8$&$8\times8$\\
        \bottomrule
    \end{tabular}
    \label{tab:ScalarTheoryTradeHparams}
\end{table}

\begin{figure}
    \centering
    \includegraphics[width=\linewidth]{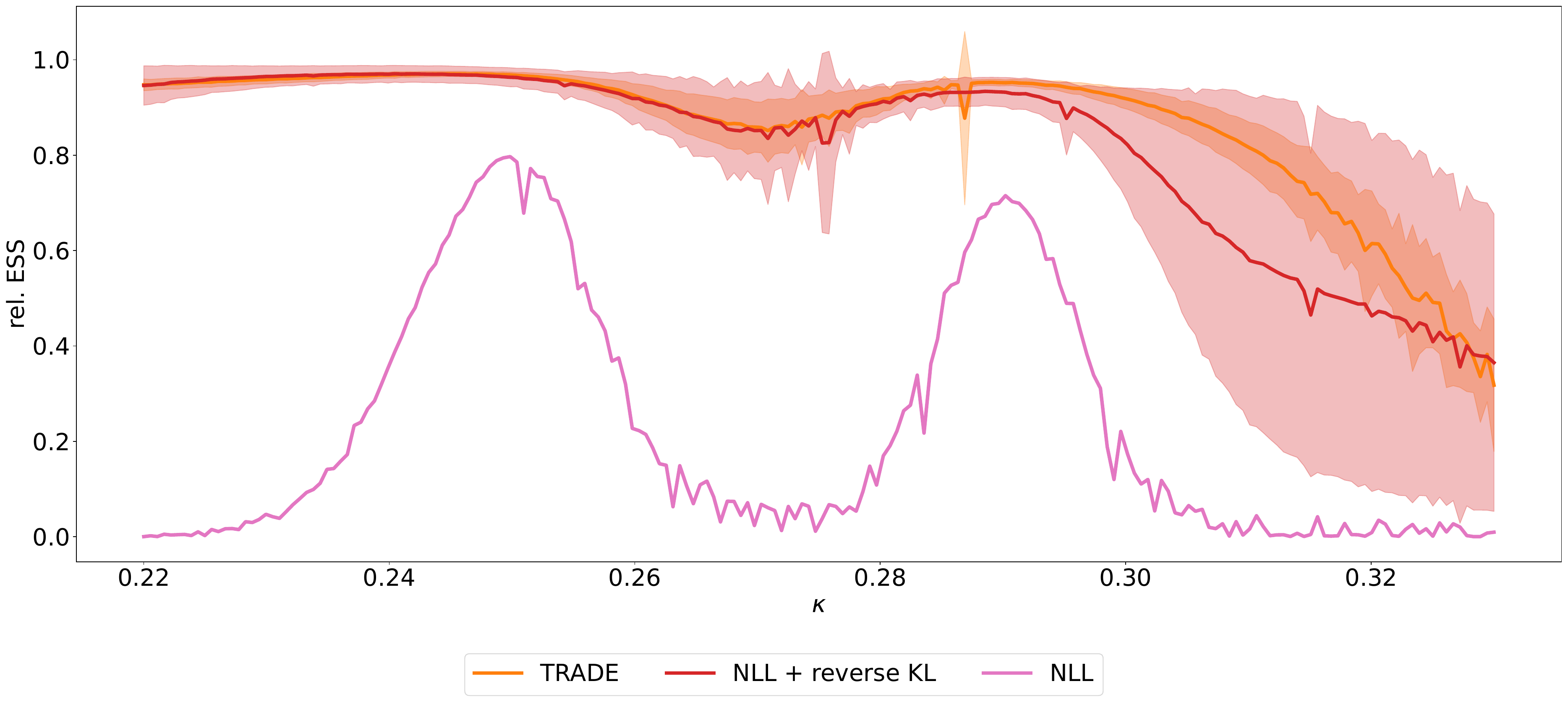}
    \caption{Relative ESS for the different training schemes examined for the scalar theory on an $8\times8$ lattice. For the model trained with TRADE and with a combination of NLL and backward KL, the solid line indicate the average over seven training runs, each initialized with a different random seed. The shaded areas correspond to the interval defined by the standard deviation of the ESS. For reference, the single run with NLL-only training is included as well. On can clearly observe that TRADE is much less sensitive to the selection of the random seed compared to the combined NLL and backward KL training.}
    \label{fig:Appendix-Scalar-Theory-ESS-different-random-seed}
\end{figure}

\section{GAUSSIAN MIXTURE MODEL}
\label{app:GMM}

We perform experiments on a Gaussian mixture model to showcase the shortcomings of several methods that were mentioned in the main text. We show the best model obtained with each training method across different temperatures in \cref{fig:Experiments_TRADE_shortcomings_density_best} and explain the experimental setup and results in the remainder of this section.

\begin{figure}
    \centering
    \includegraphics[width=0.75\linewidth]{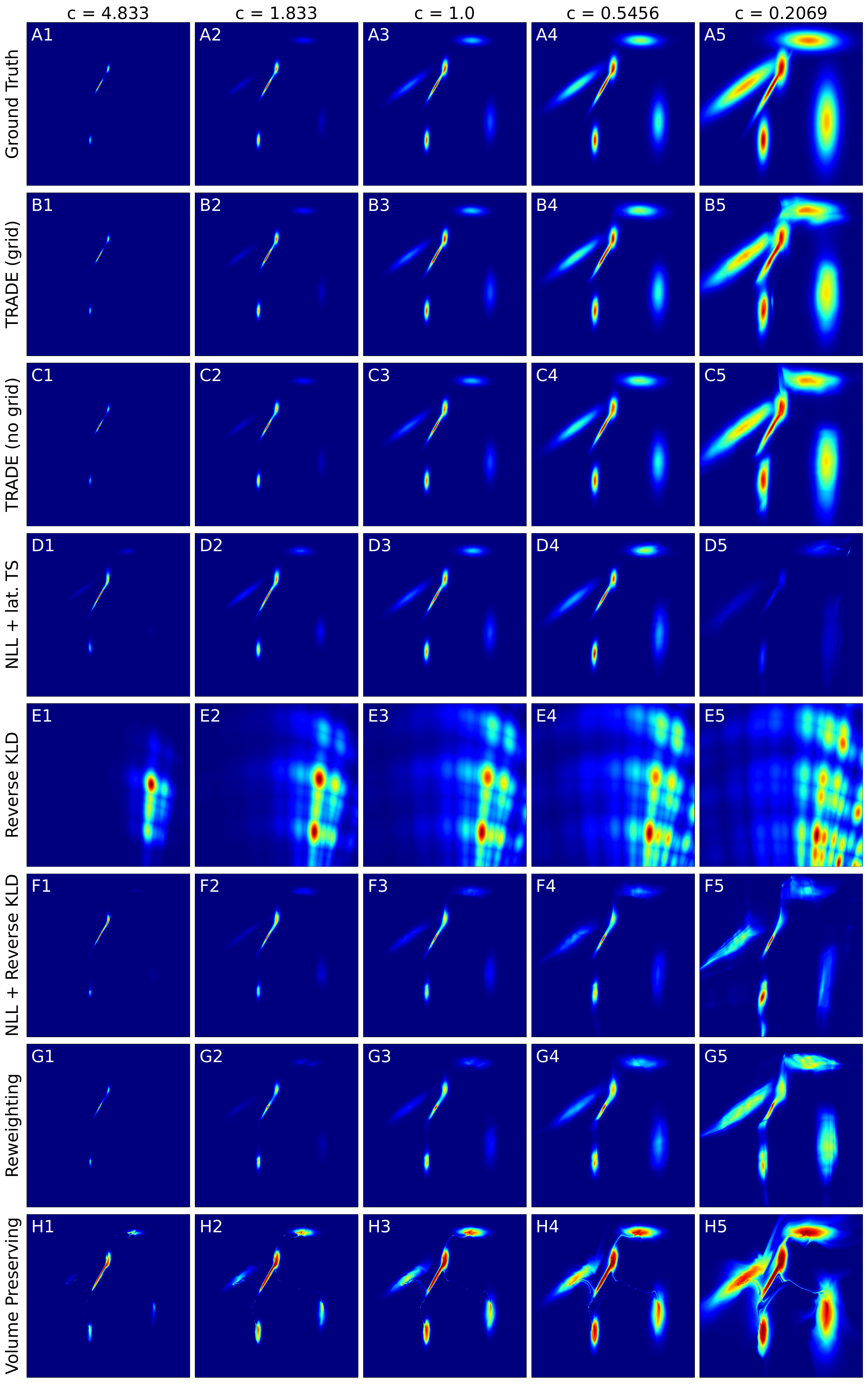}
    \caption{Visualization of the best learned densities for coupling based normalizing flows trained with different training objectives in comparison to the ground-truth density (\textbf{A}) for the best performing model based on the average validation negative log-likelihood (nll) at different external parameters. \textbf{B}: Model trained using TRADE with a discretized parameter space. \textbf{C}: TRADE trained with a continuous parameter space. \textbf{D}: Only nll training on training data at $c_0 = 1.0$ with latent power-scaling \textbf{E}: Reverse KL training \textbf{F}: Combined nll and reverse KL training. \textbf{G}: Training using reweighting of training samples at $c_0 = 1.0$ \textbf{H}: Volume-Preserving normalizing flow trained with nll training at $c_0$ and latent power-scaling.} 
    \label{fig:Experiments_TRADE_shortcomings_density_best}
\end{figure}

\subsection{Data Set}

The data set used in this experiment is a two-dimensional Gaussian mixture model with six equally weighted modes:

\begin{equation}\label{eq:Experiments_2D_GMM_base_distribution}
        p^*((x,y)|c_0) = \frac{1}{6}\sum_{i = 1}^6 \mathcal{N}\left((x,y);\mu_i,\Sigma_i\right).
\end{equation}

Power-scaling is applied to this distribution, resulting in the following conditional target distribution:

\begin{equation}\label{eq:Experiments_2D_GMM_base_distribution_ps}
        p^*((x,y)|c) \propto p^*((x,y)|c_0)^{\frac{c}{c_0}}
\end{equation}

The covariance matrices $\Sigma_i$ and the means $\mu_i$ of the six Gaussian modes of the target distribution $p^*(x|c_0)$ are chosen as follows:

\begin{equation}
    \Sigma_1 = \begin{bmatrix}
    0.2778&0.4797\\0.4797&0.8615
\end{bmatrix}\Sigma_2 = \begin{bmatrix}
    0.8958&-0.0249\\-0.0249&0.1001
\end{bmatrix}\Sigma_3 = \begin{bmatrix}
    1.3074&0.9223\\0.7744&0.1001
\end{bmatrix}\notag
\end{equation}
\begin{equation}
    \Sigma_4 = \begin{bmatrix}
    0.0305& 0.0142\\ 0.0142&0.4409
\end{bmatrix}\Sigma_5 = \begin{bmatrix}
    0.0463&0.0294\\0.0294&0.3441
\end{bmatrix}\Sigma_6 = \begin{bmatrix}
    0.1500&0.0294\\0.0294&1.5000
\end{bmatrix}\notag
\end{equation}

\begin{equation}
    \mu_1 = \begin{bmatrix}
    -1.0\\2.0 
\end{bmatrix}
    \mu_2 = \begin{bmatrix}
    3.0\\7.0 
\end{bmatrix}
    \mu_3 = \begin{bmatrix}
    -4.0\\2.0 
\end{bmatrix}
    \mu_4 = \begin{bmatrix}
    -2.0\\-4.0 
\end{bmatrix}
    \mu_5 = \begin{bmatrix}
    0.0\\4.0 
\end{bmatrix}
    \mu_6 = \begin{bmatrix}
    5.0\\-2.0 
\end{bmatrix}\notag
\end{equation}

\subsection{Models}

\subsubsection{Non-Volume-Preserving Flow}

The non-volume-preserving flows are implemented as four consecutive Rational Quadratic Spline coupling blocks \citep{durkan2019neuralsplineflows}. The logarithm of the external parameter $c$ is incorporated into each coupling block by concatenating it with the block's input. The hyperparameters of the Rational Quadratic Spline transformations are set to the default values of FrEIA's \citep{freia} implementation of the Rational Quadratic Spline coupling block. For the model trained with the negative log-likelihood method at $c_0$ only, the invertible function is not conditioned on the external parameter $c$. Therefore, in this model, no conditioning information is passed to the coupling blocks. Each coupling block is preceded by an ActNorm layer and followed by a fixed random permutation of the output dimensions. This ensures that the active and passive dimensions are alternated in each coupling block. The subnetworks used to predict the transformation parameters in each coupling block are fully connected neural networks with two hidden layers of width 32 and SiLU activation. The weights of the linear transformations in these networks are initialized using PyTorch's Xavier normal initialization scheme \citep{pmlr-v9-glorot10a}. The weights and biases of the output layer in each subnetwork are initialized to zero to ensure that the invertible function resembles the identity function at the start of the training. The latent distribution is a power-scaled standard normal distribution:

\begin{equation}
    p_0(z|c) \propto \mathcal{N}\left(z;0,\mathbf{1}_2\right)^{\frac{c}{c_0}}.
\end{equation}

\subsubsection{Volume-Preserving Flow}

The volume-preserving flow consists of 20 GIN coupling blocks \citep{sorrenson2020disentanglementnonlinearicageneral} and uses a learnable permutation after each coupling block. The subnetworks and their initialization is the same as those used for the Rational Quadratic Splines, except that ReLU activation is used instead of SiLU activation. No condition is ingested to the coupling blocks. The latent distribution is a power-scaled standard normal distribution.

\subsection{Training}

All models are trained using the Adam optimizer \citep{kingma2017adammethodstochasticoptimization} with a maximum learning rate of $1\cdot 10 ^{-4}$, modulated by a oneCycle learning rate scheduler \citep{smith2018superconvergencefasttrainingneural}. A total of $100,000$ training samples following $p(x|c_0 = 1.0)$ are used to compute the negative log-likelihood loss for all training objectives except pure reverse KL training. The batch size for the negative log-likelihood loss contribution is $256$ and the batch size for all data-free losses is set to $512$. Training runs for a total of 600 epochs. Whenever the negative log-likelihood loss is combined with a data-free loss contribution (i.e., TRADE and reverse KL combined with negative log-likelihood training) the adaptive loss balancing scheme described in \cref{app:LossBalancing} is applied. The estimates $\hat\lambda_{\text{boundary}}$ and $\hat\lambda_{\text{grad}}$ are smoothed with an exponential averaging scheme, with smoothing parameter $\alpha_{\text{balance}} = 0.015$. For all models, gradient clipping with a threshold of 2.0 is applied. For all data-free loss contributions (i.e. TRADE and reverse KL), the parameters at which the loss is computed are sampled from the interval

\begin{equation}
    \left[c_{\text{min}},c_{\text{min}}\right] = [0.16238,6.15848]
\end{equation}. 

If the negative log-likelihood loss is combined with a data-free loss, both losses are jointly optimized from the beginning of the training. Besides these general training settings, the following loss-specific hyperparameters are applied:

\subsubsection{TRADE With Grid}

The condition is discretized into a grid of 250 points, initialized uniformly in the logarithmic space. The parameter governing the decay rate in the adaptive distribution for the condition values, at which the physics-informed loss is evaluated, is set to $\epsilon = 1.0$ and is decayed to $\epsilon = 0.8$ over the course of training using an exponential decay scheme. The losses stored for each grid point are smoothed using an exponential averaging scheme with a smoothing parameter $\alpha_{\text{loss}} = 0.9$. The stored expectation values are updated every 1,000 training steps based on 1,000 samples. These expectation values are also smoothed using an exponential averaging scheme with a parameter $\alpha_{\text{expectation}} = 0.75$. Instead of the means squared loss, the physics-informed loss contribution utilizes the Huber loss with a parameter $\delta = 0.1$. In each training step, the physics-informed loss contribution is evaluated at 512 condition values $c$ sampled from the adaptive distribution $p_{\text{grad}}$. For each condition value, one evaluation point is drawn from $p_{\theta}(x|c)$.

\subsubsection{TRADE Without Grid}

The condition $c$ is sampled uniformly in the logarithmic domain, with the borders of the interval from which the parameters are sampled linearly increasing in the logarithmic domain over time, reaching the full range by the end of the training, after 600 epochs. In each training step, five different parameters $c$ are sampled, at which the physics-informed loss contribution is evaluated for 105 evaluation points following $p_{\theta}(x|c)$. For the approximation of the expectation values required in the physics-informed loss contribution, 500 samples are used for each value of $c$. The remaining parameters follow those used in the grid-based version of TRADE.

\subsubsection{Reverse KL Training}

For the computation of the reverse KL objective, 512 condition values $c$ are sampled uniformly in the logarithmic space. These samples are then used to approximate the reverse KL objective. Unlike for TRADE, where the sampling range expands gradually, here the full range of possible condition values is considered from the beginning of the training.

\subsection{Model Selection}\label{sec:Experimental_Detail_Shortcomings_Base_Lines_model_selection_TRADE}

To evaluate the performance of the model at certain stages of training, the average validation negative log-likelihood is monitored and computed every five epochs. Based on this criterion, the best-performing model is selected. For this evaluation, 10,000 data points are generated at 15 different evaluation parameters. The evaluation parameters are equally spaced in the logarithmic domain in the interval $[0.2069,4.8330]$. Since sampling from the power-scaled target distribution is not straight forward, rejection sampling is applied. To ensure that the proposal distribution has sufficient overlap with the target distribution-and that enough proposal samples are obtained at the tails of the target distribution for small values of $c$-Gaussian noise with known variance and mean is added to the proposal states drawn from the target distribution $p(x|c_0)$ (sampling from $p(x|c_0)$ is simple as it is a Gaussian mixture model). This approach corresponds to convolving each mode of $p(x|c_0)$ with a normal distribution, which results in another normal mode. This allows for an easy evaluation of the density of the convolved distribution, as required for rejection sampling.

\subsection{Results}

For the best model from each training run, the learned densities are visualized in \cref{fig:Experiments_TRADE_shortcomings_density_best}. Examining the two models trained with the TRADE objective (subplots B1 to B5 for the grid-based version of TRADE and C1 to C5 for the grid-less version of TRADE), it is evident, that the models have successfully learned the distribution from the available training data at $c_0 = 1.0$ (see subplots B3 and C3). A visual comparison between the target and the learned distributions suggests that the models have also captured how the distribution scales with the external parameter. The number of learned modes matches the target across all visualized parameter values, and the shape and relative brightness (and therefore height) are consistent with the target. However, for the largest condition value (subplots B5 and C5), a slight mismatch in the shape of the learned modes can be observed. For the model trained only with negative log-likelihood training at $c_0 = 1.0$ (subplots D1 to D5), with latent power-scaling, it can be observed that while the number and position of modes are learned correctly, their relative height deviates from the target. This is expected, as the flow is non-volume-preserving, and there is no reason to assume that simply scaling the latent distribution would result in the correct scaling in data space. The model trained with reverse KL training (subplots E1 to E5) completely fails to reproduce the characteristics of the target distribution: The number of modes is incorrect, and both their positions and their shapes are misrepresented. While the widening of the distribution as $c$ increases is evident, the relative weighting of the individual modes is inconsistent over different values of $c$. For instance, in subplots E1 and E4, the highest modes appear to swap positions. The model trained with combined reverse KL and negative log-likelihood objective (subplots F1 to F5) captures the general characteristics and scaling properties of the target. However, for higher values of $c$, the relative heights and shapes of the learned modes deviate from the target. Examining the model trained with reweighting, it can be observed that for $c < 1.0$, where the target distribution is narrower than the distribution at $c_0$ from which the training data was sampled, the model closely matches the target. However, for larger values of $c$ (e.g. subplot G5), the distribution at $c_0$ is narrower than the target at $c$, providing only little training signal in the region of the tails of the target distribution. Consequently, the learned distribution's tails are too narrow in this regime of condition values compared to the target distribution. Looking at the remaining model, the volume-preserving flow with latent power-scaling, it is evident, that as expected, the scaling  with $c$ is correctly modeled based on the distribution learned at $c_0$. However, the relative heights of the different modes are inaccurate (see, for example, subplot H1) compared to the target. Additionally, the learned modes are connected by high-density bridges (see, for example, subplot H5).\newline

This qualitative evaluation of the performance of the different training approaches is supported by analyzing the average validation (forward) KL divergence between the target distribution and the learned distribution, as presented in \cref{tab:Experiments_Shortcomings_Base_line_methods_validation_KL}. The values in this table are based on 80,000 validation samples drawn from the target distribution at each of the specified condition values. The uncertainties shown in the table are approximated using bootstrapping, with 20 resampled sets.\newline

After evaluating the \textit{best} performing mode for each training run, it is also insightful to examine the \textit{final} model of each training run. The visualization of the densities for these snapshots of the training are presented in \cref{fig:Experiments_TRADE_shortcomings_density_last}. For the two models trained with TRADE, the model trained with reweighting, the model trained with negative log-likelihood training only, and the volume-preserving flow, the densities exhibit only minor differences between \cref{fig:Experiments_TRADE_shortcomings_density_best} and \cref{fig:Experiments_TRADE_shortcomings_density_last}. However, for the two models trained using the reverse KL objective, a significant change can be observed: In \cref{fig:Experiments_TRADE_shortcomings_density_best}, the model trained with reverse KL alone shows no resemblance to the target density. In contrast, in \cref{fig:Experiments_TRADE_shortcomings_density_last}, it has learned some of the modes correctly—one mode for the three largest values of $c$, and four or five modes for the two smallest values of $c$, respectively. For the model trained with the combination of reverse KL and negative log-likelihood training, all modes are captured in \cref{fig:Experiments_TRADE_shortcomings_density_best} but are subsequently forgotten in \cref{fig:Experiments_TRADE_shortcomings_density_last}. Similar to the model trained with reverse KL alone, this issue is particularly sever for large values of $c$. The only exception is $c = 1.0$, where negative log-likelihood training was performed, ensuring that all modes are properly learned. This observation confirms the well-known issue of mode-seeking behavior in reverse KL training. Notably, for the combined reverse KL and nll training, this result suggests that providing training data does not have a stabilizing effect on the model but merely prevents mode collapse at the specific parameter value where training data is available.

\begin{figure}
    \centering
    \includegraphics[width=0.75\linewidth]{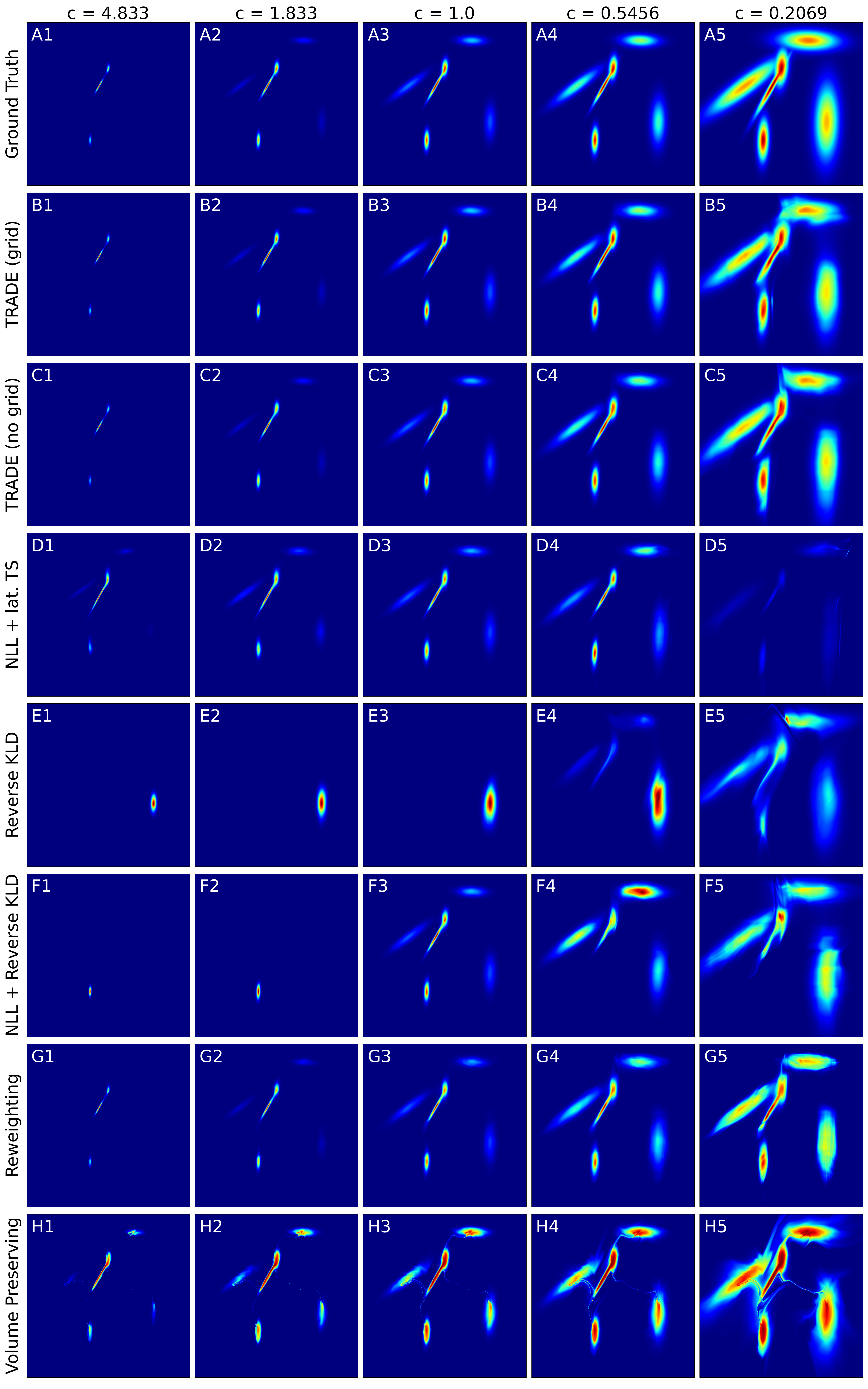}
    \caption{Visualization of the final learned densities for coupling based normalizing flows trained with different training objectives in comparison to the ground truth density (\textbf{A}) for model observed at the end of the training run evaluated at different external parameters. \textbf{B}: Model trained using TRADE with a discretized parameter space. \textbf{C}: TRADE trained with a continuous parameter space. \textbf{D}: Only nll training on training data at $c_0 = 1.0$ with latent power-scaling \textbf{E}: Reverse KL training \textbf{F}: Combined nll and reverse KL training. \textbf{G}: Training using reweighting of training samples at $c_0 = 1.0$ \textbf{H}: Volume-Preserving normalizing flow trained with nll training at $c_0$ and latent power-scaling.} 
    \label{fig:Experiments_TRADE_shortcomings_density_last}
\end{figure}

\begin{table}[]
    \centering
    %Subtable 1
    \begin{tabularx}{\textwidth}{|c|>{\centering\arraybackslash}X|>{\centering\arraybackslash}X|>{\centering\arraybackslash}X|}
    \hline
    &KLD $c = 4.8330\downarrow$&KLD $c = 2.9764\downarrow$&KLD $c = 1.8330\downarrow$\\
    \hline
    \rowcolor{lightgray}TRADE (grid)&\textbf{0.011482$\pm$0.000487}&0.009177$\pm$0.000393&0.006239$\pm$0.000407\\
    \rowcolor{lightgray}TRADE (no grid)&0.011761$\pm$0.00054&\textbf{0.00746$\pm$0.000422}&\textbf{0.005112$\pm$0.000301}\\
    NLL + lat. TS&0.58906$\pm$0.0023&0.37628$\pm$0.00279&0.14435$\pm$0.00176\\
    Reverse KLD&9.19012$\pm$0.00337&6.45853$\pm$0.00596&4.63207$\pm$0.00463\\
    NLL + Reverse KLD&0.32929$\pm$0.00299&0.229$\pm$0.0023&0.11099$\pm$0.00169\\
    Reweighting&0.11298$\pm$0.0014&0.09674$\pm$0.00122&0.09815$\pm$0.00127\\
    Volume Preserving&0.80099$\pm$0.00236&0.55106$\pm$0.00293&0.26806$\pm$0.00274\\
    \hline
    \end{tabularx}
    
    %Subtable 2
    \begin{tabularx}{\textwidth}{|c|>{\centering\arraybackslash}X|>{\centering\arraybackslash}X|>{\centering\arraybackslash}X|}
    \hline
    &KLD $c = 1.1288\downarrow$&KLD $c = 1.0\downarrow$&KLD $c = 0.8859\downarrow$\\
    \hline
    \rowcolor{lightgray}TRADE (grid)&0.003704$\pm$0.000284&0.004421$\pm$0.00029&0.004042$\pm$0.000397\\
    \rowcolor{lightgray}TRADE (no grid)&\textbf{0.002825$\pm$0.000299}&\textbf{0.003702$\pm$0.000323}&\textbf{0.003906$\pm$0.000357}\\
    NLL + lat. TS&0.022951$\pm$0.000657&0.019051$\pm$0.000589&0.032336$\pm$0.00087\\
    Reverse KLD&3.56005$\pm$0.00668&3.37188$\pm$0.00591&3.08986$\pm$0.00583\\
    NLL + Reverse KLD&0.04906$\pm$0.000899&0.049896$\pm$0.000861&0.06853$\pm$0.00115\\
    Reweighting&0.07873$\pm$0.0011&0.07396$\pm$0.00151&0.074053$\pm$0.000894\\
    Volume Preserving&0.11362$\pm$0.00201&0.09766$\pm$0.00153&0.05473$\pm$0.00172\\
    \hline
    \end{tabularx}
    
    %Subtable 3
    \begin{tabularx}{\textwidth}{|c|>{\centering\arraybackslash}X|>{\centering\arraybackslash}X|>{\centering\arraybackslash}X|}
    \hline
    &KLD $c = 0.5456\downarrow$&KLD $c = 0.3360\downarrow$&KLD $c = 0.2069\downarrow$\\
    \hline
    \rowcolor{lightgray}TRADE (grid)&0.013757$\pm$0.000584&\textbf{0.035$\pm$0.00133}&\textbf{0.08647$\pm$0.0025}\\
    \rowcolor{lightgray}TRADE (no grid)&\textbf{0.013047$\pm$0.000637}&0.03706$\pm$0.00159&0.08707$\pm$0.00146\\
    NLL + lat. TS&0.16599$\pm$0.00211&0.46467$\pm$0.0031&0.92979$\pm$0.00628\\
    Reverse KLD&2.90968$\pm$0.00406&2.78683$\pm$0.00498&2.75597$\pm$0.00506\\
    NLL + Reverse KLD&0.14625$\pm$0.00254&0.34991$\pm$0.00488&0.74761$\pm$0.00944\\
    Reweighting&0.06681$\pm$0.00152&0.16942$\pm$0.00286&0.52953$\pm$0.00739\\
    Volume Preserving&0.09508$\pm$0.00208&0.15019$\pm$0.00351&0.2335$\pm$0.00415\\
    \hline
    \end{tabularx}

    \caption{Validation KL divergence between the target distribution and the models trained with various objective functions evaluated at different external parameters. The models selected for this evaluation are the best performing models under the average validation negative log-likelihood. Bold numbers indicate the best performing model at each condition value. The rows highlighted in gray represent the entries for the models trained with the TRADE approach. In this experiment TRADE outperforms the evaluated base line methods at each evaluated external parameter.}
    \label{tab:Experiments_Shortcomings_Base_line_methods_validation_KL}
\end{table}

\end{document}